\documentclass[12pt]{article} 

\usepackage{a4wide}

\usepackage{graphics,graphicx,epsfig} 
\usepackage{amsfonts, amsmath, amssymb, wasysym}
\usepackage{stmaryrd}

\usepackage{tikz}

\usepackage{url}
\usepackage{latexsym,mathrsfs}
\usepackage{amsfonts, amsmath, amssymb}
\usepackage{graphics}
\usepackage{pstricks, pst-node, pst-tree}
\usepackage{stmaryrd}

\usepackage{url}

\newcommand{\eq}{\leftrightarrow}

\newcommand{\imp}{\rightarrow}
\newcommand{\Imp}{\Rightarrow}

\newcommand{\et}{\wedge}
\newcommand{\vel}{\vee}
\newcommand{\Et}{\bigwedge}

\newcommand{\T}{\top}

\newcommand{\Dia}{\Diamond}

\newcommand{\M}{\hat{K}}

\renewcommand{\phi}{\varphi}
\newcommand{\union}{\cup}

\newcommand{\inter}{\cap}

\newcommand{\bisim}{{\raisebox{.3ex}[0mm][0mm]{\ensuremath{\medspace \underline{\! \leftrightarrow\!}\medspace}}}}

\newcommand{\weg}[1]{}

\newcommand{\lbr}{[\![}
\newcommand{\rbr}{]\!]}
 
\newcommand{\II}[1]{\lbr #1 \rbr} 

\newcommand{\pre}{\mathsf{pre}}

\newcommand{\Formulas}{{\mathcal L}}

\newcommand{\langu}{\Formulas}
\newcommand{\States}{S}

\newcommand{\state}{s}

\newcommand{\stateb}{t}
\newcommand{\statec}{u}

\newcommand{\Atoms}{P}

\newcommand{\atom}{p}

\newcommand{\Agents}{A}
\newcommand{\Group}{B}
\newcommand{\agent}{a}

\newcommand{\agentb}{b}
\newcommand{\agentc}{c} 

\newcommand{\amodel}{\ensuremath{\mathsf{M}}}

\newcommand{\arel}{\ensuremath{\mathsf{R}}}
\newcommand{\Actions}{\ensuremath{\mathsf{S}}}
\newcommand{\actiona}{\ensuremath{\mathsf{s}}}
\newcommand{\actionb}{\ensuremath{\mathsf{t}}}

\newcommand{\Nat}{\mathbb N}
\newcommand{\Naturals}{\Nat}

\newcommand{\lang}{\langu}

\usepackage[thmmarks]{ntheorem}
\theoremsymbol{\ensuremath{\dashv}}
\usepackage{newproof}  

\newtheorem{theorem}{Theorem}

\newtheorem{definition}[theorem]{Definition}

\newtheorem{proposition}[theorem]{Proposition}

\newtheorem{conjecture}[theorem]{Conjecture}

\newcommand{\sv}{{\mathcal{S}5}}
\newcommand{\kfv}{{\mathcal{K}45}}
\newcommand{\kdfv}{{\mathcal{KD}45}}

\newcommand{\upd}[2]{#1|{#2}}

\newcommand{\ourneg}[1]{{\sim}#1}

\usepackage[all]{xy}

\newcommand{\post}{\mathsf{post}}

\usepackage[authoryear]{natbib}

\begin{document}

\title{True Lies}

\author{Thomas {\AA}gotnes\thanks{Information Science and Media Studies, University of Bergen, Bergen, Norway \& Center for the Study of Language and Cognition, Zhejiang University, China; {\tt thomas.agotnes@infomedia.uib.no}}, Hans van Ditmarsch\thanks{LORIA -- CNRS / University of Lorraine, Nancy, France; {\tt hans.van-ditmarsch@loria.fr}}, Yanjing Wang\thanks{Dept.\ of Philosophy, Peking University, Beijing, China; {\tt y.wang@pku.edu.cn}. Corresponding author.}}


\date{\today}

\maketitle

\begin{abstract}
A true lie\footnote{{\em True Lies} is a 1994 James Cameron movie featuring Arnold Schwarzenegger and Jamie Lee Curtis. It is among the very first imported Hollywood movies in China.} is a lie that becomes true when announced. In a logic of announcements, where the announcing agent is not modelled, a true lie is a formula (that is false and) that becomes true when announced. We investigate true lies and other types of interaction between announced formulas, their preconditions and their postconditions, in the setting Gerbrandy's logic of believed announcements, wherein agents may have or obtain incorrect beliefs. Our results are on the satisfiability and validity of instantiations of these semantically defined categories, on iterated announcements, including arbitrarily often iterated announcements, and on syntactic characterization. We close with results for iterated announcements in the logic of {\em knowledge} (instead of {\em belief}), and for lying as {\em private} announcements (instead of {\em public} announcements) to different agents. Detailed examples illustrate our lying concepts.
\end{abstract}

\section{Introduction}

\begin{quote}
{\em 
Pang Juan was an ancient Chinese military general of the Wei state during the Warring States period. Both he and Sun Bin studied under the tutelage of the hermit Guiguzi, but later led the opposing armies of two countries (resp.) Wei and Qi at war. Ambushed by Qi, the Wei army suffered a crushing defeat and Pang Juan committed suicide. In traditional folklore, Sun Bin carved the words ``Pang Juan dies under this tree'' on a tree at the ambush area. When Pang and his men arrived, he saw that there were carvings on the tree so he lit a torch for a closer look. At that moment, the Qi troops lying in ambush attacked and Pang Juan committed suicide under that very tree. \quad \\ \quad \hfill {(Based on \url{https://en.wikipedia.org/wiki/Pang_Juan}) }
}
\end{quote}

The Pang Juan example demonstrates that the announcement of something false may make it true. (A detailed analysis will be given later.) This is reminiscent of the logic of truthful public announcements, wherein the announcement of something true may make it false. The archetypical example is the announcement of $\atom \et \neg \Box \atom$, for example when you are being told: ``You do not know that Pang Juan died in 342 BC.'' To get $\atom \et \neg \Box \atom$ we have to read this sentence as its implicature ``Pang Juan died in 342 BC and you do not know that.''  This phenomenon may even be called the reason that this logic exists, that it is considered interesting, and various puzzles build upon the phenomenon. But in truthful public announcement logic we cannot announce something false thus making it true. The announced formulas are supposed to be true. In an alternative semantics for public announcement logic, that of conscious update \citep{gerbrandyetal:1997}, and that is also known as {\em believed (public) announcement logic}, the announcement of a formula is independent of its truth. In that logic, we can call an announcement a lie if the announced formula is false. A true lie is then the announcement of something false that makes it true. 

This analysis of a lie is not entirely according to the traditional view of agents lying to each other: you are lying to me if you believe $\neg\phi$ but you say to me that $\phi$ with the intention to make me believe that $\phi$.\footnote{This analysis of lying goes back to \cite{Augustine:dm}. Lying has been a thriving topic in the philosophical community since then \citep{siegler:1966,bok:1978,mahon:2006,mahon.stanford:2008}. More recent modal logical analyses that can be seen as a continuation of this philosophical tradition include  \cite{baltag:2002,steiner:2006,kooietal:2011,hvdetal.lying:2012,sakama:2011,liuetal:2013,hvd.lying:2014}. In modal logics with only belief operators the intentional aspect of lying is not modelled.} If we model the realization of this intention (namely that indeed I believe you), and incorporate all belief aspects, we get $\Box_a\neg\phi \imp [\Box_a\phi]\Box_b \Box_a\phi$ (where agent $a$ lies to agent $b$). Abstracting from all belief aspects gives us $\neg\phi \imp [\phi]\phi$. Although the former is more precise, it is not uncommon, if not customary, to abstract from the source and the justification of the lie so that we get $\neg\phi \imp [\phi]\Box_b \phi$, from where it is not a big step to $\neg\phi \imp [\phi]\phi$, the main focus of our investigation. (In order to make this identification, we can think of agent $b$ as the special agent who is able to distinguish all states, for whom it holds that $\Box_b \phi \eq \phi$.) This simplied schema $\neg\phi \imp [\phi]\phi$ is clearly in opposition to the schema $\phi \imp [\phi]\phi$ for the so-called {\em successful formulas}. (In truthful public announcement logic the successful formulas are those for which $[\phi]\phi$ is valid. In that logic $[\phi]\phi$ is equivalent to $\phi \imp [\phi]\phi$. But in believed public announcement logic, where announcements are independent of the truth of the announced formula, $[\phi]\phi$ is not equivalent to $\phi \imp [\phi]\phi$.) In this work we mainly investigate true lies, successful formulas, and yet other notions in believed announcement logic, and also the iteration of such announcements.

We also present results in this work that are not for believed announcement logic, but that are still related to iterated announcements, or to lying. 

In the logic of believed public announcements, interpreted on $\kdfv$ models encoding belief, the announced formula can be false and true before the announcement and can also be false and true after the announcement; and even beyond that, when iterating that same announcement, this value can change in arbitrary ways after each next announcement. In the logic of truthful public announcements, interpreted on $\sv$ models encoding knowledge, lying is not possible. But it can still be that formulas other than the announcement keep changing their value when iterating a given announcement, in ways otherwise very similar to the results for belief.

In the Pang Juan example (as we will explain later) it can be argued that the action involved is not so much an announcement but an assignment. Assignments model that propositional variables change their value (it is also known as ontic change): the proposition ``Pang Juan dies under this tree'' was false, and afterwards it is true: factual change. We present a scenario involving lying wherein private (i.e., non-public) announcements and private assignments both play a role: two friends hesitating to go to a party are both being lied to that the other one is going to the party, subsequent to which they both change their minds (factual change) and go to the party.

We now give an overview of our results and of the contents of our paper. Section~\ref{sec.two} recalls the logic of believed announcements, and the subsequent Section \ref{sec.twoplus} introduces all novel terminology to put our results in perspective, and also gives an overview of our results. These results are then developped in detail in the remaining sections. In Section \ref{sec.2valid} we present results on successful formulas, true lies, self-refuting formulas, and (so-called) impossible lies (that remain false even after announcement). In Section \ref{sec.three} we elaborate on the result that there are models and announcements such that by iterating that announcement in the model we can change its value arbitrarily often and in arbitrarily chosen ways. In Section \ref{sec.four} we syntactically characterize the single-agent true lies, i.e., formulas that satisfy $\neg\phi \imp [\phi]\phi$. In Section \ref{sec.five} we discuss iteration of truthful public announcements on models encoding knowledge, and in Section \ref{sec.six} the interaction of private announcements and private assignments, by way of the above-mentioned example of two friends going to party. Section \ref{sec.last} presents an integration of our different results in view of future research.

\section{Logic of believed announcements} \label{sec.two}

We recall the modelling of lying and truthtelling (after \citep{hvdetal.lying:2012}) in {\em believed (public) announcement logic}, also known as `arrow elimination' (not necessarily truthful) public announcement logic \citep{gerbrandyetal:1997}, which is a lesser known alternative for the better known `state elimination' (truthful) public announcement logic \citep{plaza:1989,baltagetal:1998}. Its language, structures, and semantics are as follows. Given are a finite set of agents $\Agents$ and a countable set of propositional variables $\Atoms$ (let $\agent\in\Agents$ and $\atom\in\Atoms$).

\begin{definition}[Language] \[ \lang \ \ni \ \phi ::= \atom \ | \ \neg \phi \ | \ (\phi \et \phi) \ | \ \Box_\agent \phi \ | \ [\phi]\phi \] \end{definition}
Other propositional connectives are defined by abbreviation. For $\Box_\agent \phi$, read `agent $\agent$ believes formula $\phi$'. If there is a single agent only, we may omit the index and write $\Box \phi$ instead. Agent variables are $\agent,\agentb,\agentc,\dots$. For $[\phi]\psi$, read `after public announcement of $\phi$, $\psi$'. If $\Box_\agent \neg\phi$, we say that $\phi$ is {\em unbelievable} (for $\agent$) and, consequently, if $\neg \Box_\agent \neg\phi$, for which we write $\Dia_\agent \phi$, we say that $\phi$ is {\em believable} (for $\agent$). This is also read as `agent $\agent$ considers it possible that $\phi$'. Shared belief $\Box_\Agents \phi$ (everybody believes that $\phi$) is defined as $\Et_{\agent\in\Agents} \Box_\agent \phi$. We say that {\em $\phi$ is believable} iff all agents consider $\phi$ believable, i.e., $\Et_{\agent\in\Agents} \Dia_\agent \phi$, for which we write $\Dia_\Agents \phi$ (however, this is not the dual of $\Box_\Agents\phi$: we may have that $\neg\Box_\Agents\neg\phi$ but not $\Dia_\Agents\phi$). Believability plays an important role in our setting.

\begin{definition}[Structures]
An {\em epistemic model} $M = ( \States, R, V )$ consists of a {\em domain} $\States$ of {\em states} (or `worlds'), an {\em accessibility function} $R: \Agents \imp {\mathcal P}(\States \times \States)$, where each $R(\agent)$, for which we write $R_\agent$, is an accessibility relation, and a {\em valuation} $V: \Atoms \imp {\mathcal P}(\States)$, where each $V(\atom)$ represents the set of states where $\atom$ is true. For $\state \in \States$, $(M,\state)$ is an {\em epistemic state}. \end{definition} An epistemic state is also known as a pointed Kripke model. We often omit the parentheses in $(M,\state)$. Without any restrictions we call the model class ${\mathcal K}$. The class of models where all accessibility relations are transitive and euclidean is called ${\mathcal K45}$, and if they are also serial it is called ${\mathcal KD45}$ (this class is the main focus of our investigations, because it encodes agents with consistent beliefs). The class of models where all accessibility relations are equivalence relations is ${\mathcal S5}$. Class ${\mathcal KD45}$ is said to have the {\em properties of belief}, and ${\mathcal S5}$ to have the {\em properties of knowledge}.
\begin{definition}[Semantics] \label{def.truthlyingpub}
Assume an epistemic model $M = ( \States, R, V )$.  
\[ \begin{array}{lcl}
M,\state \models \atom &\mbox{iff} & \state \in V_\atom \\ 
M,\state \models \neg \phi &\mbox{iff} & M,\state \not \models \phi \\ 
M,\state \models \phi \et \psi &\mbox{iff} & M,\state \models \phi  \text{ and } M,\state \models \psi \\  
M,\state \models \Box_\agent \phi &\mbox{iff} & \mbox{for all \ }  \stateb \in \States: R_\agent(\state,\stateb) \text{ implies } M,\stateb  \models \phi \\  
M,\state \models [\phi] \psi &\mbox{iff} & M|\phi,\state \models \psi 
\end{array} \] 
where epistemic model $M|\phi = ( \States, R^\phi, V )$ is as $M$ except that for all $\agent\in\Agents$, $R^\phi_\agent \ := \ R_\agent \inter \ (\States \times \II{\phi}_M)$ (and where $\II{\phi}_M := \{ \state\in\States \mid M,\state\models \phi\}$). 
\end{definition} 

In our semantics, $[\phi]\Box_\agent\psi$ means that, regardless of the truth of $\phi$, agent $\agent$ believes $\psi$ after public announcement of $\phi$. In particular, $[\atom]\Box_\agent\atom$ is valid: after the announcement of a propositional variable, the agent believes that it is true; regardless of the value of that variable. This is why it is called {\em believed (public) announcement} of $\phi$, in contrast to the {\em truthful (public) announcement} of $\phi$ \citep{plaza:1989}, wherein we restrict the model to the subdomain of all states where $\phi$ is true (see Appendix A). As said, believed announcement logic originates with \cite{gerbrandyetal:1997}, where it is called the logic of conscious updates. In believed announcement logic new information is accepted by the agents independent from the truth of that information. In truthful announcement logic new information is only incorporated if it is true. 

It should be noted that in believed announcement logic announcements can be truthful (namely when true) and lying (namely when false), whereas in truthful announcement logic announcements can only be truthful. In \cite{hvdetal.lying:2012} the believed public announcement of $\phi$ is modelled as non-deterministic choice between such truthful and lying public announcement of $\phi$, so that `after truthful announcement of $\phi$, $\psi$' corresponds to $\phi \imp [\phi] \psi$, and `after lying announcement of $\phi$, $\psi$' (after the lie that $\phi$, $\psi$) corresponds to $\neg\phi \imp [\phi] \psi$.

The announcing agent is not modelled in announcement logics, but only the effect of its announcements on the audience, the set of all agents. The interaction between announcement and belief can be formalized as $[\phi] \Box_\agent \psi \eq \Box_\agent (\phi \imp [\phi]\psi)$ \citep{gerbrandyetal:1997}.

Believed and truthful announcement are therefore closely related. And even in a technical sense: whenever a believed announcement is true, the semantics delivers results that are indistinguishable in the logics. This is because on any epistemic model, the model restriction semantics and the arrow restriction semantics result in bisimilar models, on the part of the model wherein the announcement is true. Investigations of these correspondences are made in \cite{kooi.jancl:2007} and \cite{hvdetal.lying:2012}. 

\section{Iteration of lying and truthful announcements} \label{sec.twoplus}

\subsection{What is a true lie?}

Let $p$ be the proposition `Pang Juan dies under this tree' (Chinese does not have a future tense).  As a consequence of Pang Juan observing `Pang Juan dies under this tree', he dies under this tree. The logic of public announcements is a logic of public observations, and observing written text is the typical action that has a straightforward formalization in logic: it corresponds to the announcement of the proposition representing that text. We could imagine Pang Juan being uncertain if he would die under this tree ($\neg \Box \atom \et \neg \Box \neg \atom$), such that observing $\atom$ removes this uncertainty.

Still, this analysis is unsatisfactory: If the written text had been  `Pang Juan does {\bf not} die under this tree', i.e., $\neg p$, he would still have died. If the written text had been `The moon is made of green cheese', i.e., the announcement of some unrelated proposition $q$, he would still have died. Even if there had nothing been written on that tree (it wasn't carvings, but merely a gnarled old tree trunk), i.e., the announcement of the trivial proposition $\T$, he would still have died. It was giving away his position by lighting the torch that caused his death, irrespective of what was written on the tree. Therefore the scenario cannot be modelled as initial uncertainty that is reduced by an informative action: it does not fit the schema $\neg \phi \imp [\phi] \phi$. 

Alternatively, with some justification it can be said that $p$ is false before the observation of $p$ but after the observation of $p$, $p$ is true. Changing the value of a variable is factual/ontic change. It could then be a lying public announcement of $p$ (observing $p$ while $p$ is false) linked to a public assignment changing the value of $p$ into true. However, this analysis has the same shortcomings as the previous one: what is observed does not matter. Also, Pang Juan can hardly be seen as enacting his own death. The actors are the soldiers of his enemy Sun Bin. Of course we can see the visual observation (informative) and the subsequent death (ontic) as two distinct, but related, actions.

Indeed \cite{peteretal:2016} argue that this action is neither epistemic (informative) nor performative (factual).

\medskip

In the True Lies movie Jamie Lee Curtis (unsuccessfully) plays the role of a timid slovenly housewife, who starts enacting the lie that she is a spy in order to make her husband jealous and thus seduce him. As Arnold Schwarzenegger really is a spy, because of her enactment she also really becomes a spy (and incidentally also a sparkling, extrovert, and well-dressed woman, a role that fits her rather well). This story of gradual and (exclusively) factual change also does not fit a logic of exclusively informational change. However, this factual changes interacts with belief (informative) change.

\medskip

In Section \ref{sec.six} we will see that the combination of (private) informational and (private) factual change after all makes for a pretty good but different kind of true lie.  With that background we will, in Section \ref{sec.last}, review once more the Pang Juan and the Arnold Schwarzenegger examples. In the current section (and the subsequent three sections, altogether the core of our contribution) we focus exclusively on true lies as informational change, i.e., $\phi$ is a true lie iff $\neg\phi \imp [\phi]\phi$ is valid in the logic of believed announcements. Also relevant is the satisfiability of $\neg\phi \et [\phi]\phi$. These different perspectives are, of course, related, and we will now proceed to explain how. (In Section \ref{sec.term} they are properly defined as special cases of a more general setup.)

\medskip

Consider the formula $\phi = p \et \Box p$ and the model $M$ consisting of two states $s$ and $t$ wherein a proposition $p$ is false and true, respectively, and that are indistinguishable for the agent. We then have that $M,t \models \neg (p \et \Box p)$  whereas $M|(p \et \Box p),t \models p \et \Box p$ so that \[ M,t \models \neg (p \et \Box p) \imp [p \et \Box p] (p \et \Box p). \] Therefore, the formula $p \et \Box p$ is a true lie in state $t$ of the model. On the other hand, although $M,s \models \neg (p \et \Box p)$, we have that $M|(p\et\Box p),s\not\models p \et \Box p$ for the simple reason that $p$ is false in $s$. So \[ M,s \not\models \neg (p \et \Box p) \imp [p \et \Box p] (p \et \Box p). \] We illustrate both transitions below:

\[ \begin{array}{lll}
\begin{tikzpicture}[thick]
\node (0) at (0,0) {$\neg p$};
\node (1) at (2,0) {$p$};
\node (0s) at (0,-0.5) {$s$};
\node (1t) at (2,-0.5) {$t$};
\draw[<->] (0) to (1);
\draw[->] (0) edge[loop above,looseness=15] (0); 
\draw[->] (1) edge[loop above,looseness=15] (1); 
\end{tikzpicture}
&
\quad \stackrel {p\et\Box p} \Imp \quad
&
\begin{tikzpicture}[thick]
\node (0) at (0,0) {$\neg p$};
\node (1) at (2,0) {$p$};
\node (0s) at (0,-0.5) {$s$};
\node (1t) at (2,-0.5) {$t$};
\end{tikzpicture}
\end{array} \]
For another example, note that \[ M|(p\et\Box p),t \models \neg (p \et \Box p) \imp [p \et \Box p] (p \et \Box p) \] for the simple reason that $M|(p\et\Box p),t \not\models \neg(p \et \Box p)$. This seems somewhat undesirable: you do not want the implication to be true because the antecedent is false. In a given model $(N,u)$ you only want to call $\phi$ a true lie if $N,u \models \neg \phi$ but $N,u \models [\phi] \phi$, in other words, if $N,u \models \neg \phi \et [\phi] \phi$. 

The formula $p \et \Box p$ is sometimes a true lie and sometimes not. Of additional interest are the formulas that are always true lies, i.e., formulas $\phi$ such that $\neg\phi \imp [\phi]\phi$ is valid. Consider model class $\kfv$. We show that $\Box p$ and $p \vel \Box p$ are true lies. However, the latter is more interesting than the former.

Let us first show that $\Box p$ is a true lie. Consider an epistemic state $(M,\state)$, where $M = (\States,R,V)$, and such that $M,\state \models \neg \Box p$. Then there is a state $\stateb$ with $(\state,\stateb) \in R$ and such that $M,\stateb \models \neg p$. Now consider the believed announcement update $\Box  p$. Let $\statec$ be any state with $(\state,\statec)\in R$. In a $\kfv$ model $M|\Box p$ with accessibility relation $R^{\Box p}$ we have that $(\state,\statec) \not \in R^{\Box p}$, because from $(\state,\stateb)\in R$ and $(\state,\statec)\in R$ follows (euclidicity) $(\statec,\stateb)\in R$ and thus $M,\statec \models \neg \Box p$, so $(\state,\statec)\not\in R^{\Box p}$. In $M|\Box p$ no state is accessible from $\state$. Thus, $M|\Box p, \state \models \Box p$ and therefore $M, \state \models \neg \Box p \imp [\Box p]\Box p$. 

If the $\Box$ modality represents belief, this is not a very interesting form of lying, as in the resulting model the agent's beliefs are inconsistent: $M|\Box p,\state\models \Box p$ but also $M|\Box p,\state\models \Box \neg p$. Differently said, as the agent did not believe $\Box p$, she does not consider it possible that the lie $\Box p$ is the truth. More interesting are what we will call the {\em believable} true lies that we associate with the validity of $(\neg \phi \et \Dia \phi) \imp [\phi]\phi$ and with the satisfiability of $(\neg \phi \et \Dia \phi) \et [\phi]\phi$ (and with the model class $\kdfv$ --- believable true lies preserve consistent beliefs). Formula $\Box p$ is not believable in the last sense, as $\neg \Box p \et \Dia \Box p$ implies (in $\kdfv$) that $\neg \Box p$ and $\Box p$ are both true, which is inconsistent.

We now show that $p \vel \Box p$ is a true lie. Consider a $\kfv$ epistemic state $(M,\state)$, where $M = (\States,R,V)$, and such that $M,\state \models \neg p \et \neg \Box p$. Again, because $M,\state \models \neg p \et \neg \Box p$ and because $M$ is a $\kfv$ model, none of the accessible states satisfy $\Box  p$. But of course, some may satisfy $p$. Therefore, $M|(p \vel \Box p), \state \models \Box p$ and thus $M|(p \vel \Box p), \state \models p \vel \Box p$. Formula $p \vel \Box p$ is therefore also a believable true lie, and even in the non-trivial sense that $\neg (p \vel \Box p) \et \Dia (p \vel \Box p) \et [p \vel \Box p] (p \vel \Box p)$ is satisfiable (namely when there are accessible $p$ worlds as well from $s$ in $M$). If $p \vel \Box p$ is believable, this lie will preserve seriality in the updated model (as there must be an accessible $p$ state), and therefore preserves $\kdfv$.

We illustrate the difference between these two true lies on the same model as above. (We recall that $p \et \Box p$ is {\em not} a true lie in both states of the model, but only in state $t$.)


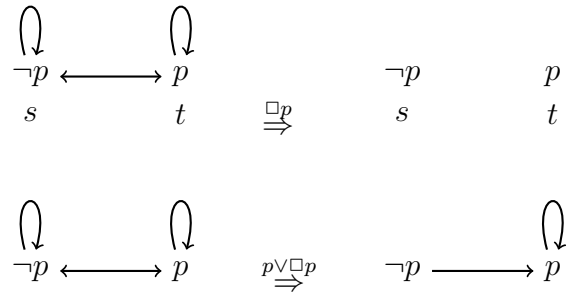
\[ \begin{array}{lll}
\begin{tikzpicture}[thick]
\node (0) at (0,0) {$\neg p$};
\node (1) at (2,0) {$p$};
\node (0s) at (0,-0.5) {$s$};
\node (1t) at (2,-0.5) {$t$};
\draw[<->] (0) to (1);
\draw[->] (0) edge[loop above,looseness=15] (0); 
\draw[->] (1) edge[loop above,looseness=15] (1); 
\end{tikzpicture}
&
\quad \stackrel {\Box p} \Imp \quad
&
\begin{tikzpicture}[thick]
\node (0) at (0,0) {$\neg p$};
\node (1) at (2,0) {$p$};
\node (0s) at (0,-0.5) {$s$};
\node (1t) at (2,-0.5) {$t$};
\end{tikzpicture} \\ \ \\ 
\begin{tikzpicture}[thick]
\node (0) at (0,0) {$\neg p$};
\node (1) at (2,0) {$p$};
\draw[<->] (0) to (1);
\draw[->] (0) edge[loop above,looseness=15] (0); 
\draw[->] (1) edge[loop above,looseness=15] (1); 
\end{tikzpicture}
&
\quad \stackrel {p \vel \Box p} \Imp \quad
&
\begin{tikzpicture}[thick]
\node (0) at (0,0) {$\neg p$};
\node (1) at (2,0) {$p$};
\draw[->] (0) to (1);
\draw[->] (1) edge[loop above,looseness=15] (1); 
\end{tikzpicture}
\end{array} \]


In this work we wish to investigate in depth which formulas may be true lies on a given model, and which formulas may be true lies on any model, and for which classes of models. In addition to true lies that become true after lying announcements, there are the successful formulas that remain true after truthful announcements, and we can consider two other options: self-refuting formulas that become false after truthful announcements, and what one might call `impossible lies' that remain false after lying announcements. We will therefore take a more general perspective: what formulas may {\em change their value or keep their value} after their announcement, how does this depend on whether they were true or false before the announcement, and what happens when we iterate these announcements? We now first introduce terminology to address these matters. 

\subsection{Terminology for iterated announcements} \label{sec.term}

Given $\sigma\in \{0,1\}^*$ and $k\leq |\sigma|$, $\sigma_k$ denotes the $k$th digit of $\sigma$ and $\sigma|k$  the prefix consisting of the first $k$ elements of $\sigma$. We abuse notation and also view $\sigma_k$ as a function $\sigma_k:\mathcal{L} \to \mathcal{L}$ on formulas such that
$$\sigma_k(\phi)=\left\{\begin{array}{ll}
\phi & \text{if } \sigma_k=1\\ 
\neg\phi & \text{if } \sigma_k=0
\end{array}\right. $$
\begin{definition} \label{def.above}
Let ${\mathcal X}$ be a class of models and $\sigma\in \{0,1\}^*$ with $n = |\sigma| \geq 2$. Formula $\phi\in\lang$ is {\em $\sigma$-satisfiable} in ${\mathcal X}$ iff $\sigma_1(\phi) \et [\phi]\tau_2(\phi)$ is satisfiable in ${\mathcal X}$, and $\phi$ is {\em $\sigma$-valid}  in ${\mathcal X}$ iff $\sigma_1(\phi) \imp [\phi]\tau_2(\phi)$ is valid in ${\mathcal X}$, where:
$$\tau_k(\phi)=\left\{\begin{array}{ll}
\sigma_k(\phi)\land [\phi]\tau_{k+1}(\phi) & 1< k<n \\
\sigma_k(\phi) & k=n
\end{array}\right.$$
Formula $\phi$ is {\em non-trivially $\sigma$-valid} iff it is $\sigma$-valid and $\sigma$-satisfiable.
 
For $\sigma \in  \{0,1\}^\omega$ (i.e.\  $\sigma: \Naturals^+ \imp \{0,1\}$), $\phi\in\lang$ is $\sigma$-satisfiable in ${\mathcal X}$ iff there is a ${\mathcal X}$ model $(M,s)$ such that $\phi$ is $\sigma|k$-satisfiable in $(M,s)$ for all $k \in \Naturals$ (with $k \geq 2$).  Formula $\phi\in\lang$ is $\sigma$-valid in ${\mathcal X}$ iff it is $\sigma|k$-valid for all $k \in \Naturals$ (with $k \geq 2$).
\end{definition} 
For example, if $|\sigma|=3$ then $\phi$ is a $\sigma$-valid formula in ${\mathcal K}$ iff  
$\sigma_1(\phi)\to [\phi](\sigma_2(\phi)\land [\phi]\sigma_3(\phi))$ is valid in ${\mathcal K}$; if $\sigma = 001$ we  get $\models \neg\phi\to [\phi](\neg\phi\land [\phi]\phi)$ (in ${\mathcal K}$). We typically only consider model classes ${\mathcal K}$ and $\kdfv$.

We can alternatively describe $\sigma$-satisfiability as follows. For that, given a model $M$, a formula $\phi\in\lang$, and $k \in \Naturals$, we first define $M|\phi^0 = M$ and $M|\phi^{k+1} = (M|\phi^k)|\phi$. Then,  $\phi\in\lang$ is $\sigma$-satisfiable in ${\mathcal X}$ iff there is a ${\mathcal X}$ model $(M,s)$ such that for all $1 \leq k \leq |\sigma|$, $M|\phi^{k-1}, s \models \sigma^k(\phi)$. 

\begin{definition}
Formula $\phi\in\lang$ is {\em believable $\sigma$-satisfiable} in ${\mathcal X}$ iff $\sigma_1(\phi) \et \Dia_\Agents\phi \et [\phi]\tau_2(\phi)$ is satisfiable in ${\mathcal X}$, and $\phi$ is {\em believable $\sigma$-valid} in ${\mathcal X}$ iff $(\sigma_1(\phi) \et \Dia_\Agents \phi) \imp [\phi]\tau_2(\phi)$ is valid in ${\mathcal X}$, where:
$$\tau_k(\phi)=\left\{\begin{array}{ll}
\sigma_k(\phi)\land \Dia_\Agents \phi\et [\phi]\tau_{k+1}(\phi) & 1< k<n \\
\sigma_k(\phi) & k=n
\end{array}\right.$$
\end{definition}
{\em Non-trivially believable $\sigma$-valid} is believable $\sigma$-valid and believable $\sigma$-satisfiable, analogously to above, in Def.~\ref{def.above}. Believable $\sigma$-satisfiable and believable $\sigma$-valid for infinite strings are also analogously defined as above. For the believable $\sigma$-satisfiable/valid formulas we typically consider model class $\kdfv$.

If $\phi$ is $\sigma$-satisfiable then it may not be believable $\sigma$-satisfiable, as the latter implies the former. If $\phi$ is $\sigma$-valid then it is also believable $\sigma$-valid, as the former implies the latter (the condition in the implication is stronger). However, if $\phi$ is non-trivially $\sigma$-valid then it may not be non-trivially believable $\sigma$-valid, and vice versa.

For strings $\sigma$ of length 2 we use corresponding English terminology that reflects common usage in the literature.
\[ \begin{array}{l|l|l|l}
\sigma & \text{\bf name} & \text{\bf $\sigma$-valid} & \text{\bf $\sigma$-satisfiable} \\ \hline
10 & \text{self-refuting} & \models \phi \imp [\phi]\neg\phi & \exists (M,s) : (M,s) \models \phi \et [\phi] \neg\phi \\ 
01 & \text{true lie} & \models \neg\phi \imp [\phi]\phi & \exists (M,s) : (M,s) \models \neg\phi \et [\phi]  \phi \\ 
11 & \text{successful} &  \models \phi \imp [\phi]\phi & \exists (M,s) : (M,s) \models \phi \et [\phi]  \phi \\ 
00 & \text{impossible lie} & \models \neg\phi \imp [\phi]\neg\phi & \exists (M,s) : (M,s) \models \neg\phi  \et [\phi]  \neg\phi \\ 
\end{array} \]
It is also common to refer to a formula that is $11$-satisfiable in a given model $(M,s)$ as a {\em successful update}, and to a formula that is not, as an {\em unsuccessful update}, etc. The successful formulas are the most well-known from the literature, and have been investigated on the class of $\sv$ models \citep{hvdetal.synthese:2006,hollidayetal:2010}. The archetypical unsuccessful update and even self-refuting formula in that setting is $\atom \et \neg \Box \atom$ \citep{moore:1942,hintikka:1962,hollidayetal:2010}.

It is easy to come up with $\sigma$-satisfiable formulas for any finite string $\sigma$. But it is even easier to see finite strings as prefixes of infinite strings, and to show that infinite $\sigma$ are always satisfiable. For infinite $\sigma$, in view of iterated relativization \citep{milleretal:2005,sadzik:2006}, formulas that can be announced arbitrarily often in a given model without stabilizing to the value of the announcement formula or its negation have our special attention. Such a formula we call {\em unstable}.

\begin{definition}[Unstable]
A $\sigma \in \{0,1\}^\omega$ is called {\em unstable} iff for all $n \in \Nat$ there are $m,k \geq n$ such that $\sigma_m = 1$ and $\sigma_k = 0$. A formula $\phi\in \lang$ is an unstable formula iff it is $\sigma$-satisfiable for an unstable $\sigma$.
\end{definition}
In other words, a formula $\phi$ is unstable iff there is a $(M,s)$ such that for all $n \in \Nat$ there are $m,k \geq n$ such that $M|\phi^m,s \models \phi$ and $M|\phi^k,s \models \neg\phi$. In the continuation we will see that there are $\sigma$-satisfiable formulas for any $\sigma \in \{0,1\}^\omega$ (this includes the unstable $\sigma$s), whereas for $\sigma$-validity we only need to consider finite strings (this does not imply stability in a particular model, but recurrence of a finite pattern).

\subsection{Overview of results and conjectures}

In this section we present results and conjectures for $\sigma$-satisfiable and $\sigma$-valid formulas, with respect to the classes ${\mathcal K}$ and $\kdfv$. We wish to encourage the reader to obtain further results. We start with two rather obvious observations. Let $\sigma$ be a finite or infinite $\{0,1\}$ string:
\begin{itemize}
\item If $\phi$ is $\sigma$-satisfiable and $\tau$ is a prefix of $\sigma$, then $\phi$ is $\tau$-satisfiable; and if $\phi$ is $\sigma$-valid and $\tau$ is a prefix of $\sigma$, then $\phi$ is $\tau$-valid.
\item If $\phi$ is believable $\sigma$-satisfiable on class $\kdfv$ in $(M,s)$, then for all $m < |\sigma|$, $M|\phi^m \in \kdfv$.
\end{itemize}
We now start with results for $\sigma$-satisfiability and continue with results for $\sigma$-validity. Various results are accompanied by lengthy proofs and drawn-out examples which are found in separate sections.

\begin{proposition} \label{prop.a}
Let $\sigma\in\{0,1\}^\omega$. Then there are $\sigma$-satisfiable formulas on class ${\mathcal K}$.
\end{proposition}

\begin{proof}
See Section \ref{sec.threeone}. We provide a single formula, namely $\neg \Box \bot \et ((\Dia \Box \bot \et \Dia \neg\Box \bot) \imp \Dia (\atom \et \Box \bot))$, and a single frame $F$ (a frame is a model without the valuation, i.e., a pair pair $(\States, R)$) with designated state $s \in \States$ such that for any $\sigma \in \{0,1\}^\omega$, by properly choosing the valuation $V$ on $F$, we can establish that $\neg \Box \bot \et ((\Dia \Box \bot \et \Dia \neg\Box \bot) \imp \Dia (\atom \et \Box \bot))$ is $\sigma$-satisfiable in $((F,V),s)$. The root $s$ is infinitely branching. 
\end{proof}

\begin{proposition} \label{prop.b}
Let $\sigma\in\{0,1\}^\omega$. Then there are believable $\sigma$-satisfiable formulas on class $\kdfv$.
\end{proposition}

\begin{proof}
See Section \ref{sec.threetwo}. The result is obtained by adapting the formula and model used in Proposition \ref{prop.a}. This is according to fairly standard method to transform a model with a directed single-agent accessibility relation into a model with two undirected accessibility relations for different agents. (The method has been used for multi-$\sv$ models but not for multi-$\kdfv$ models, to our knowledge.) 
%
\end{proof}

We continue with results on $\sigma$-validity. Our special interest have non-trivial $\sigma$-valid formulas and non-trivial believable $\sigma$-valid formulas. (We recall that `non-trivial' means that they are also $\sigma$-satisfiable.)

First, a negative and a positive result for single updates.

\begin{proposition} \label{prop.nonon}
There are no non-trivial $01$-valid formulas on class ${\mathcal K}$.
\end{proposition}

\begin{proof}
If $\phi$ is a trivial $01$-valid formula then $\phi$ is equivalent to $\T$. Let us therefore assume that $\neg\phi$ is satisfiable.

Suppose towards contradiction that such a non-trivial true lie exists. Call it $\phi$. Then $\neg\phi$ is satisfiable on a model. Without loss of generality suppose it is a tree, with the root of the tree as the designated state.

Now consider a singleton model with an empty accessibility relation, i.e., a dead end. Call it $(M,s)$. If $(M,s)$ satisfies $\neg\phi$ then $\phi$ cannot be a true lie, since the announcement of $\phi$ does not change the model (the accessibility relation is already empty). Therefore $(M,s)$ satisfies $\phi$. 

Now consider any model $(N,t)$ such that $t$ only has successors that are dead ends. Using the above result for singletons we know that those successors also satisfy $\phi$. Now if $(N,t)$ were to satisfy $\neg\phi$, then $\phi$ is not a true lie since the announcement of $\phi$ would not change the model either: all successors are $\phi$ states. Therefore $(N,t)$ has to satisfy $\phi$.

Continuing this reasoning, it will be clear that $\phi$ must be satisfied in all finite-depth trees. Therefore, if a non-trivial true lie exists, it can only be satisfiable on an infinite-depth tree.

Let $(M,s)$ be an infinite-depth tree such that $(M,s)$ satisfies $\neg\phi$ and $[\phi]\phi$. The formula $\neg\phi\et[\phi]\phi$ is equivalent to an announcement-free formula $\psi$ (the logic of believed announcement is equally expressive as the base modal logic, via a rewriting system). Let $d$ be the modal depth of $\psi$.

For $d'> d$, let $(N,t)$ be a finite-depth tree that is $d'$-bisimilar to $(M,s)$ (we refer to \cite{blackburnetal:2001} for this notion). Then for all formulas $\eta$ of modal depth at most $d'$, $(M,s) \models \eta$ iff $(N,t) \models \eta$. Therefore, as $(M,s) \models \psi$ and the modal depth of $\psi$ is $d < d'$, also $(N,t) \models \psi$, i.e., $(N,t) \models \neg\phi \et [\phi]\phi$. This contradicts that finite-depth trees cannot satisfy true lies as already obtained above.

As neither finite-depth nor infinite-depth trees may satisfy $\neg\phi$, no tree satisfies $\neg\phi$ and thus no pointed model. Therefore non-trivial true lies do not exist on class ${\mathcal K}$.
\end{proof}

\begin{proposition} 
Let $\sigma\in\{0,1\}^2$. Then there are non-trivial $\sigma$-valid formulas and non-trivial believable $\sigma$-valid formulas on class $\kdfv$.
\end{proposition}

\begin{proof}
See Section \ref{sec.2valid}.  The section also gives references for the origin of the corresponding English terminology, in the context of validity and satisfiability.
\end{proof}
We recall that a believable true lie is a believable $01$-valid formula. The following result is a syntactical characterization of this semantic notion, on the $\kdfv$ models. The term `disjunctive lying form'  is defined syntactically in Section \ref{sec.four}.
\begin{proposition}[Characterization of believable true lies] \label{prop.char} \label{th:char}
A formula is a believable true lie iff it is not equivalent to a disjunctive lying form.
\end{proposition}

\begin{proof}
See Section \ref{sec.four}. The proof closely follows the proof of the syntactic characterization of successful formulas in \cite{hollidayetal:2010}.
\end{proof}
There are successful formulas that are not true lies, such as booleans. Clearly $\models p \imp [p]p$,  whereas  $\not\models \neg p \imp [p]p$. There are also successful formulas that are true lies, such as $\neg \Box p$ and (in the universal fragment:) $p \vel \Box p$. At some stage we conjectured that {\em all true lies are successful formulas}. But this is not the case.

\begin{proposition}[Some true lies are not successful formulas] \label{prop.conjone}
On class $\kfv$ ($\kdfv$), some 01-valid formulas are not 11-valid. 
\end{proposition}
\begin{proof}
Let $\phi=\Box\bot\lor (p\land \Diamond p\land \Diamond \neg p)\lor(\neg p \land \Diamond p \land \Box p)$. We show that on $\kfv$ $\phi$ is 01-valid but not 11-valid. 

Due to transitivity and euclidicity, a pointed $\kfv$ model consists of a designated point possibly pointing at a cluster of indistinguishable points. The formula $\phi$ actually specifies its $\kfv$ models modulo bisimilarity. We list all eight $\kfv$ pointed models (modulo bisimilarity) for one variable $p$, where we distinguish those where $\phi$ is true from those where $\phi$ is false, and we also give the result of the $\phi$ update. The underlined state is the designated state.

\bigskip

\noindent --- $\phi$ is true at these four models:
\vspace{-.5cm}
\begin{center}
\begin{tikzpicture}[thick]
\node (min) at (-.7,0) {$a:$};
\node (0) at (0,0) {$\neg p$};
\node (1) at (1.5,0) {\underline{$p$}};
\end{tikzpicture}
\quad\quad\quad
\begin{tikzpicture}[thick]
\node (min) at (-.7,0) {$b:$};
\node (0) at (0,0) {\underline{$\neg p$}};
\node (1) at (1.5,0) {$p$};
\end{tikzpicture}
\quad\quad\quad
\begin{tikzpicture}[thick]
\node (min) at (-.7,0) {$c:$};
\node (0) at (0,0) {$\neg p$};
\node (1) at (1.5,0) {\underline{$p$}};
\draw[<->] (0) to (1);
\draw[->] (0) edge[loop above,looseness=15] (0); 
\draw[->] (1) edge[loop above,looseness=12] (1); 
\end{tikzpicture}
\quad\quad\quad
\begin{tikzpicture}[thick]
\node (min) at (-.7,0) {$d:$};
\node (0) at (0,0) {\underline{$\neg p$}};
\node (1) at (1.5,0) {$p$};
\draw[->] (0) to (1);
\draw[->] (1) edge[loop above,looseness=15] (1); 
\end{tikzpicture}
\end{center}
--- annoucing $\phi$ in those models delivers:
\vspace{-.5cm}
\begin{center}
\begin{tikzpicture}[thick]
\node (min) at (-.7,0) {$a':$};
\node (0) at (0,0) {$\neg p$};
\node (1) at (1.5,0) {\underline{$p$}};
\end{tikzpicture}
\quad\quad\quad
\begin{tikzpicture}[thick]
\node (min) at (-.7,0) {$b':$};
\node (0) at (0,0) {\underline{$\neg p$}};
\node (1) at (1.5,0) {$p$};
\end{tikzpicture}
\quad\quad\quad
\begin{tikzpicture}[thick]
\node (min) at (-.7,0) {$c':$};
\node (0) at (0,0) {$\neg p$};
\node (1) at (1.5,0) {\underline{$p$}};
\draw[->] (0) to (1);
\draw[->] (1) edge[loop above,looseness=12] (1); 
\end{tikzpicture}
\quad\quad\quad
\begin{tikzpicture}[thick]
\node (min) at (-.7,0) {$d':$};
\node (0) at (0,0) {\underline{$\neg p$}};
\node (1) at (1.5,0) {$p$};
\end{tikzpicture}
\end{center}
\noindent --- $\phi$ is false at these four models:
\begin{center}
\begin{tikzpicture}[thick]
\node (min) at (-.7,0) {$e:$};
\node (0) at (0,0) {$\neg p$};
\node (1) at (1.5,0) {\underline{$p$}};
\draw[->] (1) edge[loop above,looseness=12] (1); 
\end{tikzpicture}
\quad\quad\quad
\begin{tikzpicture}[thick]
\node (min) at (-.7,0) {$f:$};
\node (0) at (0,0) {\underline{$\neg p$}};
\node (1) at (1.5,0) {$p$};
\draw[->] (0) edge[loop above,looseness=12] (0); 
\end{tikzpicture}
\quad\quad\quad
\begin{tikzpicture}[thick]
\node (min) at (-.7,0) {$g:$};
\node (0) at (0,0) {\underline{$\neg p$}};
\node (1) at (1.5,0) {$p$};
\draw[<->] (0) to (1);
\draw[->] (0) edge[loop above,looseness=12] (0); 
\draw[->] (1) edge[loop above,looseness=15] (1); 
\end{tikzpicture}
\quad\quad\quad
\begin{tikzpicture}[thick]
\node (min) at (-.7,0) {$h:$};
\node (0) at (0,0) {$\neg p$};
\node (1) at (1.5,0) {\underline{$p$}};
\draw[<-] (0) to (1);
\draw[->] (0) edge[loop above,looseness=12] (0); 
\end{tikzpicture}
\end{center}
\noindent --- announcing $\phi$ in those models delivers:
\vspace{-.5cm}
\begin{center}
\begin{tikzpicture}[thick]
\node (min) at (-.7,0) {$e':$};
\node (0) at (0,0) {$\neg p$};
\node (1) at (1.5,0) {\underline{$p$}};
\end{tikzpicture}
\quad\quad\quad
\begin{tikzpicture}[thick]
\node (min) at (-.7,0) {$f':$};
\node (0) at (0,0) {\underline{$\neg p$}};
\node (1) at (1.5,0) {$p$};
\end{tikzpicture}
\quad\quad\quad
\begin{tikzpicture}[thick]
\node (min) at (-.7,0) {$g':$};
\node (0) at (0,0) {\underline{$\neg p$}};
\node (1) at (1.5,0) {$p$};
\draw[->] (0) to (1);
\draw[->] (1) edge[loop above,looseness=15] (1); 
\end{tikzpicture}
\quad\quad\quad
\begin{tikzpicture}[thick]
\node (min) at (-.7,0) {$h':$};
\node (0) at (0,0) {$\neg p$};
\node (1) at (1.5,0) {\underline{$p$}};
\end{tikzpicture}
\end{center}
Models $e',f',g',h'$ are bisimilar to $a, b, d, a$ respectively, thus $\neg \phi\to [\phi]\phi$ holds on $e, f, g, h$, i.e., $\phi$ is 01-valid. However, $c'$ is bisimilar to $e$ thus $\phi\to [\phi]\phi$ is false at model $c$, i.e., $\phi$ is not 11-valid. 

We not only have that $\phi$ is 01-valid but not 11-valid on $\kfv$ but also that $\phi$ is 01-valid (obvious) but not 11-valid (not obvious: the transition between $c$ and $c'$ is between $\kdfv$ models) on $\kdfv$. The reader may further observe that $\phi$ is also non-trivially believable 01-valid (the transition from $g$ to $g'$ is between $\kdfv$ models).
\end{proof}
%
%
A natural question to ask is whether for all $\sigma \in \{0,1\}^* \union \{0,1\}^\omega$ there are (non-trivial) $\sigma$-valid formulas. The answer is (un)fortunately negative. For example, there are no (non-trivial) $001$-valid formulas. We recall that $001$-valid means $\vDash \neg \phi\to [\phi](\neg\phi \land[\phi]\phi)$. Let $(M,s)$ be such that $M,s\nvDash \phi$. Then $M,w\vDash [\phi](\neg \phi \land [\phi]\phi)$, i.e., $M|\phi,s\vDash \neg \phi$ and $M|\phi,s\models [\phi]\phi$. From the latter we get $M|\phi|\phi,s\models \phi$. However, as $\neg \phi\to [\phi](\neg\phi \land[\phi]\phi)$ is valid it also holds on $(M|\phi, s)$, so that $M|\phi,s\models\neg \phi\to [\phi](\neg\phi \land[\phi]\phi)$, and in particular $M|\phi,s\models \neg \phi\to [\phi]\neg\phi$, so that, as $M|\phi,s\models \neg\phi$, we have $M|\phi|\phi,s\models \neg\phi$ . Contradiction.

We already observed above that if $\sigma$ is a prefix of $\tau$ and a formula is $\tau$-valid then it is also $\sigma$-valid. In the other direction this is very obviously false, for example there are impossible lies (such as $\neg\Box p$, i.e., $00$-valid formulas). But we saw in the previous paragraph that there are no $001$-valid formulas. However, sometimes, if $\sigma$ is a prefix of $\tau$ and a formula is $\sigma$-valid, then it is also $\tau$-valid. 

Given $\sigma\in\{0,1\}^\omega$ and prefix $\tau$ of $\sigma$, then $\sigma$ and $\tau$ (and $\tau$ and $\sigma$) are {\em completion equivalent} iff for all intermediate strings $\tau'$ (for all prefixes $\tau'$ of $\sigma$ such that $\tau$ is a prefix of $\tau'$), a formula is $\sigma$-valid iff it is $\tau'$-valid. This relation is clearly an equivalence relation. The shortest sequence of a completion equivalence class is the {\em representative} and the longest (possibly infinite) sequence of this class is the {\em completion}.

\begin{conjecture}
The following completion equivalence classes of $\sigma$-valid formulas are all non-empty and all different, and there are no other classes.
\[\begin{array}{llll}
 01^k \ (k>0) \quad & \quad 10^k \ (k>0) \quad  & \quad 01^k0 \ (k\geq 0) \quad & \quad  10^k1 \ (k\geq 0)
\end{array}\]
\end{conjecture}

If $\sigma = 01^k0$ (for $k\geq 0$), then $\tau = 0(1^k0)^\omega$ is its completion; and similarly if $\sigma = 10^k1$ (for $k\geq 0$), then $\tau = 1(0^k1)^\omega$ is its completion. For example, the $11$-valid formulas are also $111$-valid, or $1111$-valid, or $111\dots$-valid: they are all completion equivalent; $111\dots$ is the completion of this class and $11$ is its representative. On the other hand, $\sigma$s of shape $01^k$ and $10^k$ (for $k > 0$) may have no infinite completions and be already maximal. 

There are $011$-validities on $\kfv$. A true lie $p \vel \Box p$ is an example. 
We recall that a true lie is by definition a $01$-validity. But formula $p \vel \Box p$ is also $011$-valid, because if announced when false, it results in a $\kfv$ model satisfying $\Box p$ and such that further updates have no informative effect. Therefore, it is also $0111$-valid, and \dots and $01^\omega$-valid. Similarly, a {\em believable} true lie $p \vel \Box p$ is an example of a $011$-validity on $\kdfv$. As it is believable, this guarantees the existence of an accessible $p$ state. That state will be preserved after the $p \vel \Box p$ update, and thus the resulting model is serial. More interesting would be a $011$-validity that is not a $0111$-validity. For another example, from Proposition \ref{prop.nonon} follows that there are no $01^k$-validities on class ${\mathcal K}$. But we have not investigated systematically which of the conjectured $\sigma$-valid types are non-empty and for which classes of models.

\section{The cases 01, 11, 10, and 00} \label{sec.2valid}

In this section we restrict ourselves to validities on class $\kdfv$ (or $\sv$, included in $\kdfv$).

\weg{
@@
Typical examples of @@check!!!@@ non-trivial believable $\sigma$-valid formulas on class $\kdfv$ are: 
\[ \begin{array}{|l|l|}
\hline
00 & \neg \Box p \\ 
01 & \Box p \\
10 & p \et \neg \Box p \\
11 & p , \Box p \\
\hline 
\end{array} \]
@@
}

\paragraph{Successful formulas}

\[ \models \phi \imp [\phi]\phi \]

A successful update in some given model is a $11$-satisfiable formula (in that model). We then call a formula `successful' if it is always a successful update, i.e., if it is not merely $11$-satisfiable but $11$-valid. The typical example of a successful update is `no child steps forward' in the muddy children problem \citep{mosesetal:1986,hvdetal.puzzle:2015}. Given $k$ muddy children out of $n$ children, this is a succesful update for the first $k-2$ times that father makes his request to step forward if you know whether you are muddy, and only an unsuccessful update the $(k-1)$th time, as the $k$th time that the request is made, the muddy children step forward. The formula in question, for the example of three muddy children, is $\neg (\Box_a m_a \vel \Box_a \neg m_a) \et \neg (\Box_b m_b \vel \Box_b \neg m_b) \et \neg (\Box_c m_c \vel \Box_c \neg m_c)$ (where $m_i$ stands for ``child $i$ is muddy''). However, the archetypical unsuccessful update is $\atom \et \neg \Box \atom$ \citep{moore:1942,hintikka:1962}. The original 1940s setting is the incoherence of a person saying ``I went to the pictures [movies] last Tuesday, but I don't believe that I did'' \cite[p.\ 543]{moore:1942}. In fact, $\atom \et \neg \Box \atom$ is a self-refuting formula: after announcing it truthfully, it always becomes false.

The successful formulas are the most well-known from the literature, and have been investigated on the class of $\sv$ models \citep{hvdetal.synthese:2006,hollidayetal:2010}. On the class $\sv$, and using the (state eliminating) semantics of truthful public announcement logic, the successful formulas are those for which $\models [\phi]\phi$ (and the self-refuting formulas are then those for which $\models [\phi]\neg\phi$). On $\sv$ (but of course not in $\kdfv$),  $\phi \imp [\phi]\psi$ is equivalent to $[\phi]\psi$. A formula is successful in believed announcement logic if and only if it is successful in truthful announcement logic. This we can easily see:

In truthful announcement logic, $M,s \models [\phi]\phi$ iff ($M,s \models \phi$ implies $M|_\phi,s \models \phi$), where $M|_\phi$ denotes the submodel consisting of the states satisfying $\phi$ (see Appendix A); whereas in believed announcement logic, $M,s \models \phi \imp [\phi]\phi$ iff  ($M,s \models \phi$ implies $M|\phi,s \models \phi$). Using that, whenever the announcement formula $\phi$ is true in a model $(M,s)$, the model $(M|_\phi,s)$ resulting from state elimination is bisimilar to the model $(M|\phi,s)$ resulting from arrow elimination.
 

A characterization of single-agent successful formulas in $\sv$ is known \citep{hollidayetal:2010} (we will use it below to obtain results for true lies), but only incidental results and no characterization is known for multi-agent successful formulas. Results are that: the positive fragment is successful \citep{hvdetal.synthese:2006} (the positive or universal fragment of the language of public announcement logic is inductively defined by $\phi ::= \atom \mid \neg \atom \mid \phi \vel \phi \mid \phi \et \phi \mid \Box_\agent \phi \mid [\neg\phi]\phi$), publicly known formulas are successful (formulas $\Box^*_\Agents \phi$, where $\Box^*_\Agents$ stands for common knowledge for the set of all agents --- however note that $\Box^*_\Group \phi$ may be unsuccessful if $\Group \subset \Agents$), $\neg \Box_\agent \atom$ is successful (and non-trivially believable) \citep{qian:2002} --- an older result than \cite{hollidayetal:2010}, but subsumed by the later characterization of \cite{hollidayetal:2010}.

A strong intuitive and technical motivation for the successful formulas in truthful announcement logic is that they describe the so-called {\em substitution free fragment} of the logic. In general we do not have for dynamic epistemic logics that $\models \phi$ iff $\models \phi[p/\psi]$, for example $[p]p$ is valid but $[p \et \neg \Box p](p \et \neg \Box p)$ is invalid. But this {\em uniform substitution} holds for the fragment of the language consisting of the successful formulas \citep{hollidayetal:2013}.

\paragraph{Self-refuting formulas}

\[ \models \phi \imp [\phi]\neg\phi \]

Similarly to the successful formulas, a formula is self-refuting in believed announcement logic if and only if it is self-refuting in truthful announcement logic. As said, a well-known self-refuting formula (and that is also non-trivially believable) is $\atom \et \neg \Box\atom$. A syntactic characterization of single-agent self-refuting formulas in $\sv$ is also presented in \cite{hollidayetal:2010}, but a more general multi-agent characterization is also not known to us (researchers have been looking for this in vain for some considerable time).

\paragraph{Impossible lies}

\[ \models \neg\phi \imp [\phi]\neg\phi \]

Clearly $\top$ is an impossible lie, as this makes the antecedent $\neg\phi$ of the implication false, but this is the trivial case. Slightly more interesting impossible lies are the booleans. (And all these are non-trivially believable.) For example, $p$ is an impossible lie because $\models \neg p \imp [p] \neg p$. This is obvious, as announcements do not change the value of propositional variables. 

Another impossible lie is $\neg \Box p$. It is trivial: $\models \Box p \imp [\neg \Box p] \Box p$ holds because in any model satisfying $\Box p$ the accessibility relation of the agent is empty after the $\neg \Box p$ update. 

The formula $\neg \Box p$ is also a successful formula. No characterization results are known for impossible lies. 

\paragraph{True lies}

\[ \models \neg\phi \imp [\phi]\phi \]

Above we have already seen examples of true lies. The formula $\T$ (or any other valid formula) is a trivial true lie. Examples of single-agent true lies are:  $\Box p$, $p \vel \Box p$ (which is non-trivially believable). 

Unsatisfiable formulas are not true lies, as $[\bot]\bot$ is equivalent to $\bot$ (in believed announcement logic, not in truthful announcement logic), so that $\neg \bot \imp \bot$ is equivalent to $\bot$. 

Section \ref{sec.four} proves a syntactical characterization of single-agent believable true lies in class $\kdfv$ (Proposition \ref{prop.char}). 

\section{An unstable formula} \label{sec.three}

In this section we investigate some models and formulas where the iteration of the same update never stabilizes and continues to transform the model, forever and ever. First comes an example with a ${\mathcal K}$ model and a single agent, and an alternating $010101...$-satisfiable update (Section \ref{sec.threeone}). Then comes another example generalizing this to arbitrary boolean functions $\sigma$ and $\sigma$-satisfiability  (Section \ref{sec.threearb}). Following that comes a multi-agent example for $\kdfv$ agents (with consistent beliefs)  (Section \ref{sec.threetwo}). 

\subsection{Example of an unstable formula}  \label{sec.threeone}

{\em We demonstrate that the formula $\neg \Box \bot \et ((\Dia \Box \bot \et \Dia \neg\Box \bot) \imp \Dia (\atom \et \Box \bot))$ is unstable for the model defined below and for the boolean function $\sigma \in \{0,1\}^\omega$ (i.e.\  $\sigma: \Naturals^+ \imp \{0,1\}$) such that $\sigma_n = 1$ for $n$ even and  $\sigma_n = 0$ for $n$ odd.}

\bigskip

Consider a model $M$ in Figure \ref{fig.unstable} on the left consisting of a root $\state$ with branches of length $n$ for each positive natural number $n$. We also consider a single propositional variable $\atom$ that is false in the leaves, and from there on towards the root alternatingly true and false. We finally take the value of $\atom$ in the root to be true, but in fact that does not matter.  (In the figure, nodes where $\atom$ is true are denoted $\circ$ and nodes where $\atom$ is false are denoted $\bullet$.) 

We now consider the update of $(M,\state)$ with the formula $\phi = \neg \Box \bot \et ((\Dia \Box \bot \et \Dia \neg\Box \bot) \imp \Dia (\atom \et \Box \bot))$. The formula $\phi$ describes, in other words: ``I am not a leaf node and if I am the root (i.e., if I am the unique node that is branching) then there is a branch of length 1 wherein $p$ is true in the leaf of that branch.'' 

We take the arrow elimination update, that is, arrows pointing to states where the formula is false are deleted, and all other arrows are preserved. The result is the middle structure in Figure \ref{fig.unstable}. The underlying frames of that model and the left model in the figure are isomorphic. This we can see by rotating the middle structure counterclockwise by 45 degrees and removing the isolated (unreachable) points. That result is pictured on the right in the figure.

If we now update again with $\phi$, the original model reappears (in the strong sense that the root-generated submodel is bisimilar --- this is the $\bisim$ notation --- and even isomorphic): \[ \begin{array}{l} M,\state \models\neg\phi \\ M|\phi,\state \models\phi \\M|\phi|\phi,\state \models\neg\phi \hspace{1cm} \hfill \text{such that } (M|\phi|\phi,s) \bisim (M,s) \end{array} \]
and in general we have in this case (we recall that $M|\phi^0 = M$ and $M|\phi^{n+1} = (M|\phi^n)|\phi$)
 \[ \begin{array}{l} M|\phi^n,\state \models\phi \hspace{1cm} \hfill \text{ for $n\in\Nat$ odd} \\
M|\phi^n,\state \models\neg\phi \hfill \text{ for $n\in\Nat$ even or $n=0$} \\
(M|\phi^n,\state) \bisim (M|\phi^{n+2},\state) \hspace{5cm} \hfill \text{for all } n \in \Naturals \\
\end{array} \]

\begin{figure}
\scalebox{.88}{
\begin{tikzpicture}[->,thick]
\node (0) at (0,0) {$\circ$};
\node (11) at (-1,0) {$\bullet$};
\node (21) at (-.85,.85) {$\circ$};
\node (22) at (-1.7,1.7) {$\bullet$};
\node (31) at (0,1) {$\bullet$};
\node (32) at (0,2) {$\circ$};
\node (33) at (0,3) {$\bullet$};
\node (41) at (.85,.85) {$\circ$};
\node (42) at (1.7,1.7) {$\bullet$};
\node (43) at (2.55,2.55) {$\circ$};
\node (44) at (3.4,3.4) {$\bullet$};
\node (0a) at (0,-.3) {\color{white} $\neg\phi$};
\node (21a) at (-1.15,.85) {$\phi$};
\node (31a) at (0.3,1) {$\phi$};
\node (32a) at (0.3,2) {$\phi$};
\node (41a) at (1.15,.85) {$\phi$};
\node (42a) at (2,1.7) {$\phi$};
\node (43a) at (2.85,2.55) {$\phi$};
\node (dots) at (3,0) {$ \ $};
\draw[->] (0) to (11);
\draw[->] (0) to (21);
\draw[->] (21) to (22);
\draw[->] (0) to (31);
\draw[->] (31) to (32);
\draw[->] (32) to (33);
\draw[->] (0) to (41);
\draw[->] (41) to (42);
\draw[->] (42) to (43);
\draw[->] (43) to (44);
\draw[dotted,-] (0) to (dots);
\end{tikzpicture}
{$\stackrel {\phi!} \Imp$}
\begin{tikzpicture}[->,thick]
\node (0) at (0,0) {$\circ$};
\node (11) at (-1,0) {$\bullet$};
\node (21) at (-.85,.85) {$\circ$};
\node (22) at (-1.7,1.7) {$\bullet$};
\node (31) at (0,1) {$\bullet$};
\node (32) at (0,2) {$\circ$};
\node (33) at (0,3) {$\bullet$};
\node (41) at (.85,.85) {$\circ$};
\node (42) at (1.7,1.7) {$\bullet$};
\node (43) at (2.55,2.55) {$\circ$};
\node (44) at (3.4,3.4) {$\bullet$};
\node (0a) at (0,-.3) {$\phi$};
\node (31a) at (0.3,1) {$\phi$};
\node (41a) at (1.15,.85) {$\phi$};
\node (42a) at (2,1.7) {$\phi$};
\node (dots) at (3,0) {$ \ $};
\draw[->] (0) to (21);
\draw[->] (0) to (31);
\draw[->] (31) to (32);
\draw[->] (0) to (41);
\draw[->] (41) to (42);
\draw[->] (42) to (43);
\draw[dotted,-] (0) to (dots);
\end{tikzpicture}
{$\bisim$}
\begin{tikzpicture}[->,thick]
\node (0) at (0,0) {$\circ$};
\node (11) at (-1,0) {$\circ$};
\node (21) at (-.85,.85) {$\bullet$};
\node (22) at (-1.7,1.7) {$\circ$};
\node (31) at (0,1) {$\circ$};
\node (32) at (0,2) {$\bullet$};
\node (33) at (0,3) {$\circ$};
\node (41) at (.85,.85) {$\bullet$};
\node (42) at (1.7,1.7) {$\circ$};
\node (43) at (2.55,2.55) {$\bullet$};
\node (44) at (3.4,3.4) {$\circ$};
\node (0a) at (0,-.3) {$\phi$};
\node (21a) at (-1.15,.85) {$\phi$};
\node (31a) at (0.3,1) {$\phi$};
\node (32a) at (0.3,2) {$\phi$};
\node (41a) at (1.15,.85) {$\phi$};
\node (42a) at (2,1.7) {$\phi$};
\node (43a) at (2.85,2.55) {$\phi$};
\node (dots) at (3,0) {$ \ $};
\draw[->] (0) to (11);
\draw[->] (0) to (21);
\draw[->] (21) to (22);
\draw[->] (0) to (31);
\draw[->] (31) to (32);
\draw[->] (32) to (33);
\draw[->] (0) to (41);
\draw[->] (41) to (42);
\draw[->] (42) to (43);
\draw[->] (43) to (44);
\draw[dotted,-] (0) to (dots);
\end{tikzpicture}
}
\caption{Example unstable formula}
\label{fig.unstable}
\end{figure}
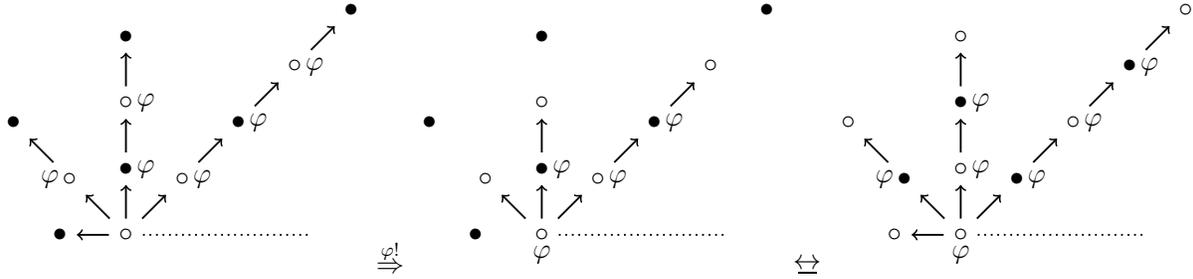

\subsection{Arbitrary boolean function} \label{sec.threearb}

Let now any boolean function $\sigma \in \{0,1\}^\omega$ be given. Consider again the frame of the model $M$ of Subsection \ref{sec.threeone} and the same formula $\phi := \neg \Box \bot \et ((\Dia \Box \bot \et \Dia \neg\Box \bot) \imp \Dia (\atom \et \Box \bot))$. The value in the root of the model is irrelevant. In that model, $\circ$ no longer means that $\atom$ is true. Instead, we decorate the model with values for atom $\atom$ according to the string $\sigma$, as below. This ensures that for all $n \in \Nat$, $M|\phi^n,s \models \sigma_{n+1}(\phi)$. 

For an example, let be $\sigma = 011100011\dots$. Then $\sigma_1 = 0$ (the first digit), $\sigma_2 = 1$, etc. Therefore $p$ is false at the leaf of the branch of length 1 (false means value 0), and in the branch of length 2 $p$ is true halfway ($\sigma_3$) and again false at the leaf, and so on. For $n=0$, we should get $M|\phi^0,s \models \sigma_1(\phi)$. As $M|\phi^0 = M$, and $\sigma_1(\phi) = \neg\phi$, we get $M,s \models \neg\phi$. And this is indeed the case, because the subformula $\Dia (\atom \et \Box \bot)$ of $\phi$ is false in the root, because $\atom$ is false ($\sigma_1 = 0$) at the leaf of the branch of length 1.

Announcement of $\phi$ causes all leaves to be eliminated from the model, as before, with the result that all leaves are now decorated with propositions $p$ with value $\sigma_2 = 1$. So we should have $M|\phi,s \models \phi$, i.e., $M|\phi^1,s \models \sigma_2(\phi)$. And so it is, because $M|\phi,s \models \Dia (\atom \et \Box \bot)$. And so on.

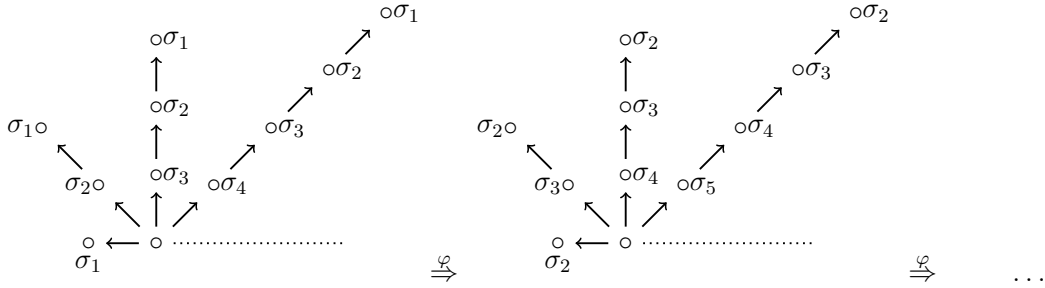
\begin{center}
\scalebox{.9}{
\begin{tikzpicture}[->,thick]
\node (0) at (0,0) {$\circ$};
\node (11) at (-1,0) {$\circ$};
\node (21) at (-.85,.85) {$\circ$};
\node (22) at (-1.7,1.7) {$\circ$};
\node (31) at (0,1) {$\circ$};
\node (32) at (0,2) {$\circ$};
\node (33) at (0,3) {$\circ$};
\node (41) at (.85,.85) {$\circ$};
\node (42) at (1.7,1.7) {$\circ$};
\node (43) at (2.55,2.55) {$\circ$};
\node (44) at (3.4,3.4) {$\circ$};
\node (11a) at (-1,-.3) {$\sigma_1$};
\node (21a) at (-1.15,.85) {$\sigma_2$};
\node (22a) at (-2,1.7) {$\sigma_1$};
\node (31a) at (0.3,1) {$\sigma_3$};
\node (32a) at (0.3,2) {$\sigma_2$};
\node (33a) at (0.3,3) {$\sigma_1$};
\node (41a) at (1.15,.85) {$\sigma_4$};
\node (42a) at (2,1.7) {$\sigma_3$};
\node (43a) at (2.85,2.55) {$\sigma_2$};
\node (44a) at (3.7,3.4) {$\sigma_1$};
\node (dots) at (3,0) {$ \ $};
\draw[->] (0) to (11);
\draw[->] (0) to (21);
\draw[->] (21) to (22);
\draw[->] (0) to (31);
\draw[->] (31) to (32);
\draw[->] (32) to (33);
\draw[->] (0) to (41);
\draw[->] (41) to (42);
\draw[->] (42) to (43);
\draw[->] (43) to (44);
\draw[dotted,-] (0) to (dots);
\end{tikzpicture}$\stackrel {\phi} \Imp$
\begin{tikzpicture}[->,thick]
\node (0) at (0,0) {$\circ$};
\node (11) at (-1,0) {$\circ$};
\node (21) at (-.85,.85) {$\circ$};
\node (22) at (-1.7,1.7) {$\circ$};
\node (31) at (0,1) {$\circ$};
\node (32) at (0,2) {$\circ$};
\node (33) at (0,3) {$\circ$};
\node (41) at (.85,.85) {$\circ$};
\node (42) at (1.7,1.7) {$\circ$};
\node (43) at (2.55,2.55) {$\circ$};
\node (44) at (3.4,3.4) {$\circ$};
\node (11a) at (-1,-.3) {$\sigma_2$};
\node (21a) at (-1.15,.85) {$\sigma_3$};
\node (22a) at (-2,1.7) {$\sigma_2$};
\node (31a) at (0.3,1) {$\sigma_4$};
\node (32a) at (0.3,2) {$\sigma_3$};
\node (33a) at (0.3,3) {$\sigma_2$};
\node (41a) at (1.15,.85) {$\sigma_5$};
\node (42a) at (2,1.7) {$\sigma_4$};
\node (43a) at (2.85,2.55) {$\sigma_3$};
\node (44a) at (3.7,3.4) {$\sigma_2$};
\node (dots) at (3,0) {$ \ $};
\draw[->] (0) to (11);
\draw[->] (0) to (21);
\draw[->] (21) to (22);
\draw[->] (0) to (31);
\draw[->] (31) to (32);
\draw[->] (32) to (33);
\draw[->] (0) to (41);
\draw[->] (41) to (42);
\draw[->] (42) to (43);
\draw[->] (43) to (44);
\draw[dotted,-] (0) to (dots);
\end{tikzpicture}
$\stackrel {\phi} \Imp \hspace{1cm} \dots$
}
\end{center}

\noindent 
This proves Proposition \ref{prop.a} that there are $\sigma$-satisfiable formulas on class ${\mathcal K}$ for any string $\sigma$.

\subsection{An unstable formula for agents with consistent beliefs} \label{sec.unstablecons}  \label{sec.threetwo}

A common way to transform a model for a directed single-agent (unimodal) setting exhibiting a certain desirable property into a model for a symmetric multi-agent setting with the same property, is to replace an arrow by a pair of (undirected) links for different agents, i.e., a pair $(x,y) \in R$ by two pairs $(x,z) \in R_a$ and $(z,y) \in R_b$.  In a multi-agent $\sv$ setting, this means chaining equivalence classes of size two for different agents $a$ and $b$. Instead of arbitrarily long finite chains, we then build arbitrarily long alternating and interlocking $a/b$ chains. The technique has been used to great effect for undecidability arguments for $\sv$ logics \citep{frenchetal:2008} and also to great effect for expressivity arguments for such logics \citep{hvdetal.del:2007,kooi.jancl:2007}. This chaining procedure can be extended in a natural way to $\kdfv$ models (for consistent belief), namely when (at least) the root of the model is an unreachable state wherein incorrect belief is possible, but where (almost) all remaining states constitute equivalence classes (wherein belief is correct). We apply this procedure to get a symmetric multi-agent model from an asymmetric single-agent model, to the model $M$ of Subsection \ref{sec.threeone}, in order to get an example of an unstable formula in a multi-agent $\kdfv$ setting. 

Starting with the model $M$ of that section, we thus get the following model $N$. The accessibility relation is solid for agent $\agent$ and it is dashed for agent $\agentb$. Intuitively, to get $N$ from $M$, all nodes are replaced by two nodes with the same valuation and that are indistinguishable for an agent $\agentb$, whereas one might say the previous directed arrows are replaced by indistinguishability for an agent $\agent$.  Except for the root, wherein we do something special: all pairs in  the accessibility relation for $\agent$ starting in the root are directed, and we add another node to the model that is accessible from the root for agent $\agentb$ only (in the figure: the node below the root), and from where all nodes accessible from the root for agent $\agent$ are also $\agent$-accessible (in the figure we have only drawn one of those arrows, for visual clarity). This is a $\kdfv$ model. We assume the usual visual conventions for such models: all directed arrows point to {\em clusters} of indistinguishable nodes for that agent (see Appendix B). However, merely in order to emphasize the similarity with the model $M$, from the root of the model we have many outgoing arrows (whereas according to the visual convention one would have been enough).

\bigskip

\scalebox{0.8}{
\begin{tikzpicture}[thick]
\node (0) at (0,0) {$\circ$};
\node (11) at (-1,0) {$\bullet$};
\node (12) at (-2,0) {$\bullet$};
\node (21) at (-.85,.85) {$\circ$};
\node (22) at (-1.7,1.7) {$\circ$};
\node (23) at (-2.55,2.55) {$\bullet$};
\node (24) at (-3.4,3.4) {$\bullet$};
\node (31) at (0,1) {$\bullet$};
\node (32) at (0,2) {$\bullet$};
\node (33) at (0,3) {$\circ$};
\node (34) at (0,4) {$\circ$};
\node (35) at (0,5) {$\bullet$};
\node (36) at (0,6) {$\bullet$};
\node (41) at (.85,.85) {$\circ$};
\node (42) at (1.7,1.7) {$\circ$};
\node (43) at (2.55,2.55) {$\bullet$};
\node (44) at (3.4,3.4) {$\bullet$};
\node (45) at (4.25,4.25) {$\circ$};
\node (46) at (5.1,5.1) {$\circ$};
\node (47) at (5.95,5.95) {$\bullet$};
\node (48) at (6.8,6.8) {$\bullet$};
\node (dots) at (3,0) {$ \ $};
\draw[->] (0) to (11);
\draw[dashed,-] (11) to (12);
\draw[->] (0) to (21);
\draw[dashed,-] (21) to (22);
\draw[-] (22) to (23);
\draw[dashed,-] (23) to (24);
\draw[->] (0) to (31);
\draw[dashed,-] (31) to (32);
\draw[-] (32) to (33);
\draw[dashed,-] (33) to (34);
\draw[-] (34) to (35);
\draw[dashed,-] (35) to (36);
\draw[->] (0) to (41);
\draw[dashed,-] (41) to (42);
\draw[-] (42) to (43);
\draw[dashed,-] (43) to (44);
\draw[-] (44) to (45);
\draw[dashed,-] (45) to (46);
\draw[-] (46) to (47);
\draw[dashed,-] (47) to (48);
\draw[-] (11) to (21);
\draw[-] (21) to (31);
\draw[-] (31) to (41);
\node (x) at (1,0.05) {$ \ $};
\draw[dotted,-] (41) to (x);
\node (newroot) at (0,-1) {$\bullet$};
\draw[dashed,->] (0) to (newroot);
\draw[->] (newroot) to (11);
\draw[dotted,-] (0) to (dots);
\end{tikzpicture}
}

\noindent Instead of \[ \phi = \neg \Box \bot \et ((\Dia \Box \bot \et \Dia \neg\Box \bot) \imp \Dia (\atom \et \Box \bot)) \] we then get a more or less corresponding multi-agent formula \[ \psi = \neg\Dia_b(\Box_a \atom \vel \Box_a \neg \atom) \et ((\atom\et\Box_b\neg\atom) \imp \Dia_a\Dia_b \Box_a \atom)\] To get the latter from the former, first replace each $\Dia$ by $\Dia_a\Dia_b$. We recall that in the model, we replaced each directed unlabeled arrow by an $\agent$ step plus a $\agentb$ step. Then, replace $\Box \bot$ by $\Dia_b(\Box_a \atom \vel \Box_a \neg \atom)$. We recall that $\Box \bot$ defined a leaf node in $M$. As $\kdfv$ models are serial, we now observe that the leaves of model $N$ are the unique nodes wherein $\agent$ knows whether $\atom$: $\Box_a \atom \vel \Box_a \neg \atom$, such that  $\Dia_b (\Box_a \atom \vel \Box_a \neg \atom)$ is only true in the final two nodes of each branch. 

The extra node is needed to allow for a formula that distinguishes the root node from other nodes: the root of $N$ is the unique node wherein agent $b$ has incorrect belief: $p \et \Box_b \neg p$. In all other nodes, agent $b$ (correctly) knows the value of $p$. (We did not know how to distinguish the root node from other nodes without introducing such an additional node.)

The nodes $a$-accessible from the root form an infinite $a$-equivalence class. The $a$ arrows leading from the extra node below the root to that class are there so that this node is not a leaf.
 
With the `translation' so far, the final part $\Dia (\atom \et \Box \bot)$ of the formula $\phi$ becomes $\Dia_a\Dia_b(\atom \et \Dia_b(\Box_a \atom \vel \Box_a \neg \atom))$, which is $\kdfv$ equivalent to $\Dia_a\Dia_b\Box_a \atom$, as found in the formula $\psi$. 

In plain English, $\psi$ describes that: $(i)$ $b$ does not consider it possible that $a$ knows whether $p$, and that $(ii)$ in the node where $b$'s belief is incorrect (the root note), $a$ considers it possible that $b$ considers it possible for $a$ to know {\em that} (not whether!) $\atom$. Condition $(i)$ is always false in the leaves and in the nodes just below the leaves, whereas condition $(ii)$ is only true in the root if there is a path of length 2 to a leaf node satisfying $p$. If we now update model $N$ with $\psi$ we get this model $N|\psi$ --- wherein $\psi$ is now true in the root, as the length 2 path is now to a $p$ node:

\bigskip

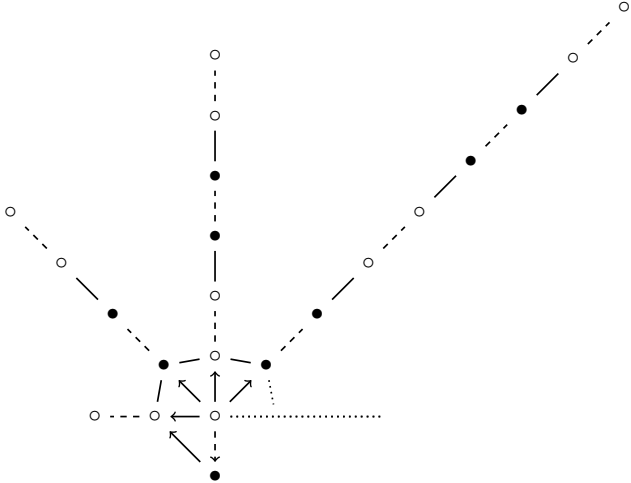
\scalebox{.8}{
\begin{tikzpicture}[thick]
\node (0) at (0,0) {$\circ$};
\node (11) at (-1,0) {$\circ$};
\node (12) at (-2,0) {$\circ$};
\node (21) at (-.85,.85) {$\bullet$};
\node (22) at (-1.7,1.7) {$\bullet$};
\node (23) at (-2.55,2.55) {$\circ$};
\node (24) at (-3.4,3.4) {$\circ$};
\node (31) at (0,1) {$\circ$};
\node (32) at (0,2) {$\circ$};
\node (33) at (0,3) {$\bullet$};
\node (34) at (0,4) {$\bullet$};
\node (35) at (0,5) {$\circ$};
\node (36) at (0,6) {$\circ$};
\node (41) at (.85,.85) {$\bullet$};
\node (42) at (1.7,1.7) {$\bullet$};
\node (43) at (2.55,2.55) {$\circ$};
\node (44) at (3.4,3.4) {$\circ$};
\node (45) at (4.25,4.25) {$\bullet$};
\node (46) at (5.1,5.1) {$\bullet$};
\node (47) at (5.95,5.95) {$\circ$};
\node (48) at (6.8,6.8) {$\circ$};
\node (dots) at (3,0) {$ \ $};
\draw[->] (0) to (11);
\draw[dashed,-] (11) to (12);
\draw[->] (0) to (21);
\draw[dashed,-] (21) to (22);
\draw[-] (22) to (23);
\draw[dashed,-] (23) to (24);
\draw[->] (0) to (31);
\draw[dashed,-] (31) to (32);
\draw[-] (32) to (33);
\draw[dashed,-] (33) to (34);
\draw[-] (34) to (35);
\draw[dashed,-] (35) to (36);
\draw[->] (0) to (41);
\draw[dashed,-] (41) to (42);
\draw[-] (42) to (43);
\draw[dashed,-] (43) to (44);
\draw[-] (44) to (45);
\draw[dashed,-] (45) to (46);
\draw[-] (46) to (47);
\draw[dashed,-] (47) to (48);
\draw[dotted,-] (0) to (dots);
\draw[-] (11) to (21);
\draw[-] (21) to (31);
\draw[-] (31) to (41);
\node (x) at (1,0.05) {$ \ $};
\draw[dotted,-] (41) to (x);
\node (newroot) at (0,-1) {$\bullet$};
\draw[dashed,->] (0) to (newroot);
\draw[->] (newroot) to (11);
\draw[dotted,-] (0) to (dots);
\end{tikzpicture}
}

\noindent The next iteration restores $N$: $(N|\psi|\psi,s) \bisim (N,s)$. Again, we get:
\[ \begin{array}{l} N|\psi^n,\state \models\psi \hspace{1cm} \hfill \text{ for $n\in\Nat$ odd} \\
N|\psi^n,\state \models\neg\psi \hfill \text{ for $n\in\Nat$ even or $n=0$} \\
(N|\psi^n,\state) \bisim (N|\psi^{n+2},\state) \hspace{5cm} \hfill \text{for all } n \in \Naturals \\
\end{array} \]

\noindent This completes our example of a believable $(01)^\omega$-satisfiable formula for $\kdfv$. From this example now directly follows Proposition \ref{prop.b} that there are believable $\sigma$-satisfiable formulas on class $\kdfv$ for any string $\sigma$: this is similar to how we generalized the example of $(01)^\omega$-satisfiability on ${\mathcal K}$ to a proof of $\sigma$-satisfiability on ${\mathcal K}$ for any $\sigma\in\{0,1\}^\omega$, using the same formula and the same model except for a different decoration of the variable $p$.

\section{Characterization of single-agent believable true lies} \label{sec.four}

Exactly which formulas are true lies? Can they been characterized syntactically, and if so, how? The corresponding problem for successful and self-refuting formulas has been studied and solved for the single-agent case \citep{hollidayetal:2010}. These characterizations are given for $\sv$ (see Section \ref{sec.2valid}) but are said to hold also for $\kdfv$ ``with minor changes'' \citep{hollidayetal:2010}. The problem is still open for the multi-agent case, for both classes of formulas. Two main reasons that the single-agent case is easier, are that in that case, first, it is well known that any formula is ($\sv$ or $\kdfv$) equivalent to a formula in disjunctive normal form without any nestings of modalities, and, second, that satisfiability of a formula in disjunctive normal form can be checked syntactically \citep{hollidayetal:2010} (see below for details). Based on these two properties, the characterizations in \cite{hollidayetal:2010} check that certain disjuncts (conditionally) exist in the (disjunctive normal form) formula, which is shown to ensure that the formula is of the required type.

In this section we give a similar characterization of believable true lies. It is heavily inspired by \cite{hollidayetal:2010}. Indeed, we take the technique directly from them, as well as notation and concepts. Just as \cite{hollidayetal:2010} syntactically characterize the set of successful formulas, we syntactically characterize the \emph{complement} of the set of believable true lies. The differences are in the details of the characterization and the corresponding proof.

In the rest of this section, a ``formula'' always refers (single-agent) formula in the standard modal language. (Note that any formula in $\lang$ is equivalent to such an announcement-free formula.) We now formally define the two properties mentioned above, disjunctive normal forms and syntactic satisfiability checking, and then give the characterization in Definition \ref{def:syntactic}.

\begin{definition}[Disjunctive normal form] A formula $\phi$ is in \emph{disjunctive normal form} iff it is a disjunction of conjunctions of the form \[\delta = \alpha \wedge \Box \beta_1 \wedge \ldots \wedge \Box \beta_n \wedge \Diamond \gamma_1 \wedge \ldots \wedge \Diamond \gamma_m\] where $\alpha$ and each $\gamma_i$ are conjunctions of literals and each $\beta_i$ is a disjunction of literals. 
\end{definition}
We write $\ourneg{l}$ to denote the negation of a literal $l$ where double negations are removed. Similarly, when $\beta$ is a disjunction of literals, $\ourneg{\beta}$ denotes the conjunction of negated literals with single negation removed. For a given conjuction $\delta$ of the form above, in the following we will write $\delta^\alpha$ for $\alpha$ and $\delta^{\Box\Diamond}$ for $\Box \beta_1 \wedge \ldots \wedge \Box \beta_n \wedge \Diamond \gamma_1 \wedge \ldots \wedge \Diamond \gamma_m$. Every formula is $\kdfv$-equivalent to one in disjunctive normal form.

\begin{definition}[Clarity \citep{hollidayetal:2010}] 
Given a conjunction or disjunction $\chi$ of literals, $L(\chi)$ denotes the set of  literals occurring in it. Set $L(\chi)$ is \emph{open} iff no literal in $L(\chi)$ is the negation of any other. A conjunction $\delta = \alpha \wedge \Box \beta_1 \wedge \ldots \wedge \Box \beta_n \wedge \Diamond \gamma_1 \wedge \ldots \wedge \Diamond \gamma_m$ in disjunctive normal form is \emph{clear} iff (i) $L(\alpha)$ is open; (ii) there is an open set of literals $\{l_1,\ldots,l_n\}$ with $l_i \in L(\beta_i)$; and (iii) for every $\gamma_k$ there is a set of literals $\{l_1,\ldots,l_n\}$ with $l_i \in L(\beta_i)$ such that $\{l_1,\ldots,l_n\} \cup L(\gamma_k)$ is open. A disjunction in disjunctive normal form is clear iff at least one of the disjuncts is clear.
\end{definition}
The key point about clarity is that a formula is $\kdfv$-satisfiable iff it is clear \cite[Lemma 3.6]{hollidayetal:2010}.

We define a formula that we call a \emph{disjunctive lying form}. We will then show that, modulo logical equivalence, the class of these formulas corresponds exactly to the class of formulas that are not believable true lies. The characterization will be explained after the definition. 

\begin{definition}[Disjunctive lying form]
  \label{def:syntactic}
 A formula $\phi$ in disjunctive normal form is a \emph{disjunctive lying form} iff there exist (possibly empty) sets $S$ and $T$ of disjuncts of $\phi$ and a conjunct $\Box \beta_\theta$ of each $\theta \in T$, such that any disjunctive normal form of
\[\chi = \neg \phi \wedge \Diamond \phi \wedge \chi_1 \wedge \chi_2 \wedge \chi_3\]
is clear\footnote{The definition allows both $S$ and $T$ to be empty. However, if $S$ is empty then (any disjunctive form of) $\chi$ is not clear ($\chi_3$ is a contradiction), so it follows that a disjunctive lying form actually must have a non-empty $S$ as a witness. If $T$ is empty then $\chi_3$ is a tautology. There exist disjunctive lying forms with an empty $T$ as a witness (one example is $p \wedge \Box p$).}
, where
\[\chi_1 = \bigwedge_{\theta \in T} t(\theta) \wedge \bigwedge_{\theta \in \overline{T}} \neg t(\theta)\quad\quad\quad t(\theta) = \theta^\alpha \wedge \bigwedge_{\Diamond \gamma \mbox{ in } \theta} \bigvee_{\sigma \in S} \Diamond(\sigma^\alpha \wedge \gamma)\]
\[\chi_2 = \bigwedge_{\sigma \in S}\sigma^{\Box\Diamond} \wedge \bigwedge_{\sigma \in \overline{S}}\neg \sigma^{\Box\Diamond}\]
\[\chi_3 = \bigwedge_{\theta \in T}\bigvee_{\sigma \in S} \Diamond(\sigma^\alpha \wedge \ourneg{\beta_\theta})\]
and $\overline{X}$ denotes the set of disjuncts of $\phi$ that are not in $X$.
\end{definition}

We now recall (page \pageref{th:char}):

\bigskip

\noindent {\bf Proposition \ref{th:char}} \ \ 
{\em A formula is a believable true lie iff it is not equivalent to a disjunctive lying form.}

\bigskip

\noindent Before we prove Proposition \ref{th:char}, let us explain the intuition behind the definition. This will also serve as a guide to the proof of Proposition \ref{th:char}. If $\chi$ is true in some pointed model (is clear), then $\neg \phi \wedge \Diamond\phi$ is true. The role of $\chi_1-\chi_3$ is to ensure that $\neg\phi$ remains false in the updated model --- which holds iff $\phi$ is not a believable true lie --- by ensuring that certain disjuncts (conditionally) exist in $\phi$.

The role of $S$ is to syntactically encode which states are still accessible from the current state after the update. These are exactly the states where $\sigma^\alpha$ is true for some $\sigma \in S$.  Since $\phi$ is false in the initial pointed model, there is only one way it can become true in the updated pointed model: if, for some disjunct $\theta$ of $\phi$ where $\theta^\alpha$ is already true and all $\Diamond\gamma$ are already true and \emph{stay} true in the updated model, all $\Box\beta$ that were false in the initial pointed model \emph{become} true in the updated model. We must avoid that. Those disjuncts that can potentially become true are exactly those in $T$. Thus, we must make sure that for every disjunct $\theta \in T$, there is at least one $\Box\beta$ such that $\beta$ is false in at least one state that is still accessible after the update. That is ensured by $\chi_3$.

As an example, to see that the believable true lie $p \vee \Box p$ (already in disjunctive normal form) is not a disjunctive lying form, we need to check for all possible subsets of disjuncts $S$ and $T$ over the set of all disjuncts $\{p, \Box p\}$, whether $\chi$ is clear. There are only two possibilities for $T$: $T = \emptyset$ or $T = \{\Box p\}$. In the former case we get that $\chi_1 = \neg p \wedge \bot$ (not clear). In the latter case, there are four possibilities for $S$: if $S=\emptyset$ then $\chi_3 = \bot$; if $S=\{p\}$ then $\chi_3 = \Diamond(p \wedge \neg p)$; if $S=\{\Box p\}$ then $\chi_2 = \Box p$ and $\chi_3 = \Diamond \neg p$; and if $S = \{p,\Box p\}$ then $\chi_2 = \Box p$ and $\chi_3 = \Diamond \neg p \wedge \Diamond (p \wedge \neg p)$. In all cases $\chi$ is not clear.

As another example, consider $p \wedge \Box p$ which is \emph{not} a believable true lie. To see that it is a disjunctive lying form, take $S = \{p \wedge \Box p\}$ and $T = \emptyset$. We get that $\chi_1 = \neg p$, and $\chi_2 = \Box p$ and $\chi_3 = \top$. It is easy to see that $\chi$ is clear (it is equivalent to $\neg p \wedge \Box p \wedge \Diamond (p \wedge \Box p)$).

\begin{proof}[of Prop.~\ref{th:char}]  Without loss of generality, let $\phi$ be a disjunctive normal form.
 
Consider first the implication towards the left: assume that $\phi$ is a disjunctive lying form and let $S$, $T$, $\chi$ and $\beta_\theta$ for each $\theta \in T$ be as in Def. \ref{def:syntactic}. We must show that there exists $M,s$ such that $M,s \models \neg \phi \wedge \Diamond \phi$ and $\upd{M}{\phi},s \models \neg \phi$.  Since any disjunctive normal form of $\chi$ is clear iff it is satisfiable, so let $M,s$ be such that $M,s \models \chi$. Then $M,s \models \neg \phi \wedge \Diamond \phi$. Assume, towards a contradiction, that $\upd{M}{\phi},s \models \phi$. Then there exists a disjunct $\theta$ in $\phi$ such that $\upd{M}{\phi},s\models \theta$.  As explained above, we are going to show that $\chi_3$ ensures that this is impossible. To do that, we must first show that it must be the case that $\theta \in T$ (intuitively defined as the set of disjuncts that can potentially become true in the updated model), and then that $\chi_3$ ensures that the conjunct $\Box\beta_\theta$ is false in $\upd{M}{\phi},s$, leading to a contradiction.

\medskip

Thus, we first show that it must be the case that $\theta \in T$. From $\upd{M}{\phi},s\models \theta^\alpha$ it follows that $M,s\models \theta^\alpha$ ($\theta^\alpha$ is propositional). We now show that also $M,s \models \bigwedge_{\Diamond \gamma \mbox{ in } \theta} \bigvee_{\sigma \in S} \Diamond(\sigma^\alpha \wedge \gamma)$, and it follows from $M,s \models \chi_1$ that $\theta \in T$. Let $\Diamond\gamma$ be a conjunct in $\theta$. We have that $\upd{M}{\phi},s \models \Diamond\gamma$, i.e., that there exists a state $t$ such that $Rst$ and $\upd{M}{\phi},t\models \gamma$ and, since $t$ is a state in $\upd{M}{\phi}$, $M,t \models \phi$. From the latter it follows that $M,t \models \sigma$ for some disjunct $\sigma$ of $\phi$.  From $M,t \models \sigma^{\Box\Diamond}$ and $Rst$ it follows that $M,s \models \sigma^{\Box\Diamond}$ by standard $\kdfv$ reasoning, and thus that $\sigma \in S$ by $\chi_2$.  Since $\gamma$ is propositional, $\upd{M}{\phi},t\models \gamma$ implies that $M,t \models \gamma$. Thus, $M,s \models \Diamond(\sigma^\alpha \wedge \gamma)$ for some $\sigma \in S$. Since $\Diamond\gamma$ was arbitrary, we get that $M,s \models \bigwedge_{\Diamond \gamma \mbox{ in } \theta} \bigvee_{\sigma \in S} \Diamond(\sigma^\alpha \wedge \gamma)$ and thus that $\theta \in T$.

Having shown that $\theta \in T$, we now get a contradiction from the fact that $M,s \models \chi_3$. It follows that there exists a $\sigma \in S$ such that $M,s \models \Diamond(\sigma^\alpha \wedge \ourneg{\beta_\theta})$. That means that there exists a state $t$ such that $Rst$ and $M,t \models \sigma^\alpha \wedge \ourneg{\beta_\theta}$. From the fact that $\sigma \in S$ we know that $M,s \models \sigma^{\Box\Diamond}$ and thus, by standard $\kdfv$ reasoning, that $M,t \models \sigma^{\Box\Diamond}$. That means that $M,t \models \sigma$ and thus $M,t \models \phi$. The latter means that $t$ is accessible from $s$ in $\upd{M}{\phi}$, and thus that $\upd{M}{\phi},s \models \neg\Box \beta_\theta$. But this contradicts $\upd{M}{\phi},s \models \Box \beta_\theta$ which follows from the assumption that $\upd{M}{\phi},s \models \theta$. This concludes the implication towards the left.

For the direction towards the right, assume that $\phi$ is not a true lie, i.e., that there exist $M,s$ such that $M,s \models \neg\phi \wedge \Diamond \phi$ and $\upd{M}{\phi},s \models \neg \phi$. We now define $T$, $S$, and $\beta_\theta$ for each $\theta \in T$, as follows. We will then show that the resulting $\chi$ is satisfiable. For any disjunct $\delta$ in $\phi$, let
  \[\delta \in T \Leftrightarrow M,s \models t(\delta) \quad\quad\quad \delta \in S \Leftrightarrow M,s \models \delta^{\Box\Diamond}\]
Let $\theta \in T$. In order to define $\beta_\theta$, we first show that for every conjunct $\Diamond \gamma$ in $\theta$, we have that $\upd{M}{\phi},s \models \Diamond \gamma$.  Let $\Diamond \gamma$ be a conjunct in $\theta$. From the fact that $M,s \models t(\theta)$ we get that there is a $\sigma \in S$ such that $M,s \models \Diamond(\sigma^\alpha \wedge \gamma)$. In other words, there exists a state $t$ such that $Rst$ and $M,t \models \sigma^\alpha \wedge \gamma$. From $\chi_2$ and the fact that $\sigma \in S$ and standard $\kdfv$ reasoning we have that $M,t \models \sigma^{\Box\Diamond}$. Thus, $M,t \models \sigma$. That means that $M,t \models \phi$ and that $t$ is accessible from $s$ in $\upd{M}{\phi}$, so $\upd{M}{\phi},s \models \Diamond \gamma$, for all $\Diamond \gamma$ in $\theta$. Since $\theta^\alpha$ is propositional, $\upd{M}{\phi},s \models \theta^\alpha$ follows from $M,s\models t(\theta)$, and thus $\upd{M}{\phi},s \models \neg \theta$ (which follows from $\upd{M}{\phi},s \models \neg \phi$) implies that there must be a conjunct $\Box\beta$ in $\theta$ such that $\upd{M}{\phi},s \models \neg \Box\beta$. Let $\beta_\theta$ be one such $\beta$.

Finally, we show that $\chi$ is satisfiable. It follows that any disjunctive normal form of $\chi$ is clear, and thus that $\phi$ is a disjunctive lying form. We show that $M,s \models \chi$ (where $M,s$ is the pointed model above). We have that $M,s \models \neg\phi \wedge \Diamond\phi$ by assumption, and $M,s \models \chi_1 \wedge \chi_2$ by definition. It remains to be shown that $M,s \models \chi_3$. Let $\theta \in T$.  By definition of $\beta_\theta$, $\upd{M}{\phi},s \models \Diamond \ourneg{\beta_\theta}$. That means that there exists a state $t$ such that $t$ is accessible from $s$ in the updated model and $\upd{M}{\phi},t \models \ourneg{\beta_\theta}$. Since $t$ is accessible from $s$ in $\upd{M}{\phi}$, $M,t \models \phi$. Let $\sigma$ be a disjunct of $\phi$ such that $M,t \models \sigma$. Thus $M,t \models \sigma^{\Box\Diamond}$, and from that and standard $\kdfv$ reasoning, $M,s \models \sigma^{\Box\Diamond}$, and thus $\sigma \in S$. From $M,t \models \sigma$ we know that $M,t \models \sigma^\alpha$, and since $\ourneg{\beta_\theta}$ is propositional, $M,t \models \ourneg{\beta_\theta} \wedge \sigma^\alpha$ for some $\sigma \in S$ and thus $M,s \models \Diamond(\ourneg{\beta_\theta} \wedge \sigma^\alpha)$.  Since $\theta \in T$ was arbitrary, $M,s \models \chi_3$.
%
\end{proof}

\section{Knowledge and iterated announcement whether} \label{sec.five}

The focus of our story is on lying and therefore on the believed announcement logic, not on the truthful announcement logic wherein lying is impossible. However some of our questions are also meaningful in a logic of knowledge change, interpreted on models with equivalence relations. As an example, let us address the matter of arbitrarily often iterated updates and unstable formulas. For $\kdfv$ models we could affirmatively answer the following question (Section \ref{sec.threetwo}):
\begin{quote} {\em 
Are there a model $(M,s)$ and a formula $\phi$, such that if we continue to announce in this model that $\phi$, the value of $\phi$ never stabilizes? \hfill (i)}
\end{quote}
This cannot be in $\sv$, as $\phi$ cannot be truthfully announced when $\phi$ is false in the actual state. An obvious question in the setting of $\sv$ models and {\em truthful} public announcements is:
\begin{quote} {\em 
Are there a model $(M,s)$ and a formula $\phi$, such that if we continue to announce in this model \pmb{whether} $\phi$, the value of $\phi$ never stabilizes? \hfill (ii)}
\end{quote}
The (truthful public) {\em announcement whether} $\phi$ announces the {\em value} of $\phi$. If $\phi$ is true, it is a truthful announcement of $\phi$, whereas if $\phi$ is false, it is a truthful announcement of $\neg\phi$. To question (ii) we do not know the answer. We conjecture that the answer is: {\bf no}. However, a related question is:
\begin{quote} {\em 
Are there a model $(M,s)$ and formulas $\phi$ \pmb{and} $\pmb{\psi}$, such that if we continue to announce whether $\phi$, the value \pmb{of} $\pmb{\psi}$ never stabilizes? \hfill (iii)}
\end{quote}
This question we can answer positively, and we think that it may be found in various epistemic scenarios of interest. In the scenario we present here, we do slightly better than (iii):
\begin{quote} {\em 
Are there a model $(M,s)$ and formulas $\phi$ and $\psi$, such that if we continue to announce \pmb{that} $\phi$, the value of $\psi$ never stabilizes?  \hfill (iv)}
\end{quote}
An answer to (iv) is also an answer to (iii).

Lack of stability in the sense of (iv) is more formally defined as follows. Let $\sigma \in \{0,1\}^\omega$ be {\em unstable}; we recall that this means that for all $n \in \Nat$ there are $m,k \geq n$ such that $\sigma_m = 1$ and $\sigma_k = 0$. Then we wish to find $(M,s)$, $\phi$, and $\psi$ such that for all $n \in \Naturals$, $M|_{\phi^n},s \models \sigma_{n+1}(\psi)$, and where $M$ is an $\sv$ model. (Where we further recall that $\sigma_n(\psi) = \psi$ if $\sigma_n=1$ and else $\sigma_n(\psi) =\neg\psi$; and that $M|_\phi$ is the ($\sv$ preserving) state elimination semantics for truthful public announcement---see Section \ref{sec.two} and Appendix A.)

Now consider $\sigma = (01)^\omega$ ($\sigma = 010101...$), the alternating string. Our solution is a simple adjustment of the model $N$ from Section \ref{sec.threetwo} to realize an unstable formula for believed announcement logic. The model below is the equivalence closure (reflexive, symmetric, and transitive closure) of that model $N$. To achieve that visually, we only have to replace the directed arrows involving the root and the node below the root by undirected links, and assume transitivity. The alternating values of $\sigma_i$ have been made explicit in three branches, for an example. As in Section \ref{sec.threetwo}, we decorate the states of $N$ with the variable $p$ according to $\sigma$; $\bullet$ means that $\atom$ is false ($\sigma_1 = \sigma_3 = 0$), and $\circ$ that $\atom$ is true ($\sigma_2 = 1$). 

\bigskip

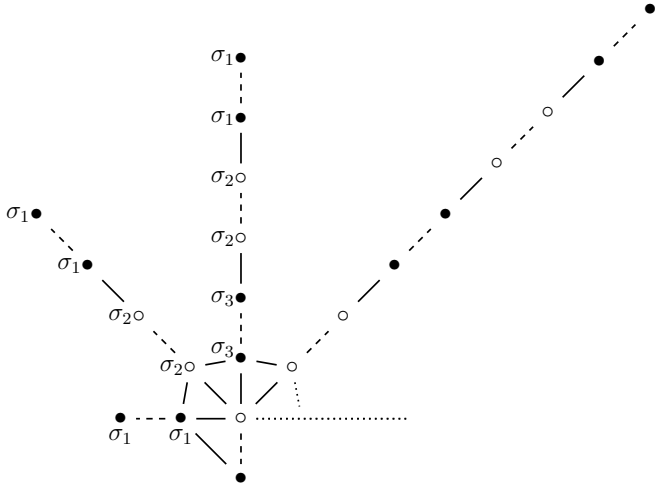
\scalebox{.8}{
\begin{tikzpicture}[thick]
\node (0) at (0,0) {$\circ$};
\node (11) at (-1,0) {$\bullet$};
\node (12) at (-2,0) {$\bullet$};
\node (21) at (-.85,.85) {$\circ$};
\node (22) at (-1.7,1.7) {$\circ$};
\node (23) at (-2.55,2.55) {$\bullet$};
\node (24) at (-3.4,3.4) {$\bullet$};
\node (31) at (0,1) {$\bullet$};
\node (32) at (0,2) {$\bullet$};
\node (33) at (0,3) {$\circ$};
\node (34) at (0,4) {$\circ$};
\node (35) at (0,5) {$\bullet$};
\node (36) at (0,6) {$\bullet$};
\node (41) at (.85,.85) {$\circ$};
\node (42) at (1.7,1.7) {$\circ$};
\node (43) at (2.55,2.55) {$\bullet$};
\node (44) at (3.4,3.4) {$\bullet$};
\node (45) at (4.25,4.25) {$\circ$};
\node (46) at (5.1,5.1) {$\circ$};
\node (47) at (5.95,5.95) {$\bullet$};
\node (48) at (6.8,6.8) {$\bullet$};
\node (11a) at (-1,-.3) {$\sigma_1$};
\node (12a) at (-2,-.3) {$\sigma_1$};
\node (21a) at (-1.15,.85) {$\sigma_2$};
\node (22a) at (-2,1.7) {$\sigma_2$};
\node (23a) at (-2.85,2.55) {$\sigma_1$};
\node (24a) at (-3.7,3.4) {$\sigma_1$};
\node (31a) at (-.3,1.15) {$\sigma_3$};
\node (32a) at (-.3,2) {$\sigma_3$};
\node (33a) at (-.3,3) {$\sigma_2$};
\node (34a) at (-.3,4) {$\sigma_2$};
\node (35a) at (-.3,5) {$\sigma_1$};
\node (36a) at (-.3,6) {$\sigma_1$};
\node (dots) at (3,0) {$ \ $};
\draw[-] (0) to (11);
\draw[dashed,-] (11) to (12);
\draw[-] (0) to (21);
\draw[dashed,-] (21) to (22);
\draw[-] (22) to (23);
\draw[dashed,-] (23) to (24);
\draw[-] (0) to (31);
\draw[dashed,-] (31) to (32);
\draw[-] (32) to (33);
\draw[dashed,-] (33) to (34);
\draw[-] (34) to (35);
\draw[dashed,-] (35) to (36);
\draw[-] (0) to (41);
\draw[dashed,-] (41) to (42);
\draw[-] (42) to (43);
\draw[dashed,-] (43) to (44);
\draw[-] (44) to (45);
\draw[dashed,-] (45) to (46);
\draw[-] (46) to (47);
\draw[dashed,-] (47) to (48);
\draw[dotted,-] (0) to (dots);
\draw[-] (11) to (21);
\draw[-] (21) to (31);
\draw[-] (31) to (41);
\node (x) at (1,0.05) {$ \ $};
\draw[dotted,-] (41) to (x);
\node (newroot) at (0,-1) {$\bullet$};
\draw[dashed,-] (0) to (newroot);
\draw[-] (newroot) to (11);
\draw[dotted,-] (0) to (dots);
\end{tikzpicture}
}

\bigskip

In order to get $N|_{\phi^n},s \models \sigma_{n+1}(\psi)$ we decorate as before the model with values $\sigma_n$ for the propositional variable $\atom$ (where the model above demonstrates strict alternation in the first three values) and we take $\phi := \neg\Dia_b(\Box_a \atom \vel \Box_a \neg \atom)$ and $\psi := \neg(\Box_b\atom\vel\Box_b\neg\atom) \imp \Dia_a\Dia_b \Box_a \atom$. In other words, we split the unstable formula $\neg\Dia_b(\Box_a \atom \vel \Box_a \neg \atom) \et ((\atom\et\Box_b\neg\atom) \imp \Dia_a\Dia_b \Box_a \atom)$ that we employed in Section \ref{sec.threetwo} into a part doing the job of always removing the leaf and the node before the leaf of every branch, and another part checking whether there is a branch of length two from the root to a $\atom$ node, where the distinguishing formula for the root has become $\agentb$'s ignorance of $\atom$ (there are now two nodes that can be the root).

\bigskip

For a different example of $M|_{\phi^n},s \models \sigma_{n+1}(\psi)$, consider an infinite number of (possibly) muddy children, of which an infinite subset are in fact muddy, let $\phi$ be the usual proposition `nobody steps forward' (nobody knows whether he or she is muddy; see Section \ref{sec.2valid}), and let $\psi$ be the proposition `it is common knowledge that at least $n$ children are muddy but it is not common knowledge that at least $n+1$ children are muddy, where $n$ is even', and let $\sigma = (01)^\omega$, again. After father's initial announcement that at least one child is muddy (which serves as the initial model $N$ in this case), $\psi$ is false (because $n=1$ which  is odd). Whereas after nobody steps forward following the first request by father to do so if you know whether you are muddy, $\psi$ is true (because now it is common knowledge that at least two children must be muddy: $n=2$ which is even). And so on, ad infinitum.

\bigskip

If we consider non-public events, there are yet other interesting forms of unstable iteration on $\sv$ models. Running ahead of the next Section \ref{sec.six} (and see Appendix A): there are {\em action models} such that their iterated execution does not stabilize on a given model. The following is a vintage example in a two-agent setting. 

Consider a two-state epistemic model for two agents Anne (solid access) and Bob (dashed access) such that Anne is uncertain whether $\atom$ but Bob knows whether $\atom$. Actually, $\atom$ is true (the $\circ$ state, boxed). Iterated execution of the action model with two actions with preconditions $\neg \Box_\agent \atom$ and $\atom \et \Box_\agentb \neg \Box_\agent \atom$ transforms this two-state epistemic model into a three-state model, and vice versa, ad infinitum. We could think of the formula $\Box_\agentb \neg \Box_\agent \atom$ being alternatingly true and false in the $\atom$-world (the $\circ$-world) about which agent $\agent$ is uncertain. The execution can be visualized as follows. 

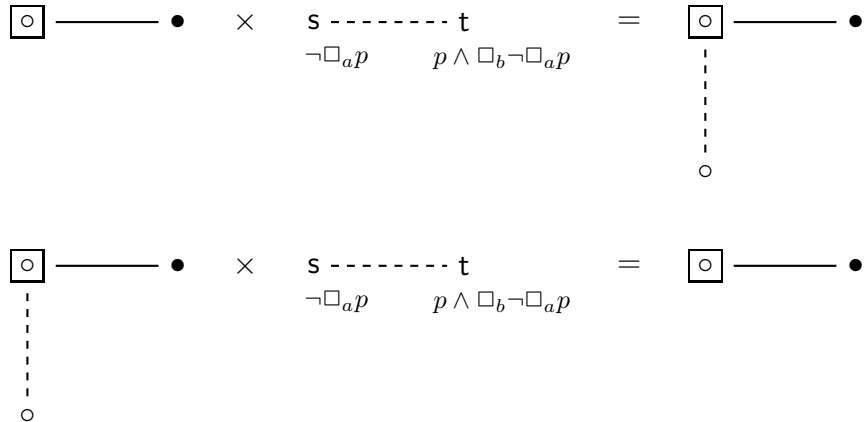
\[ \begin{array}{lllll}
\begin{tikzpicture}[thick]
\node (0) at (0,0) {\framebox{$\circ$}};
\node (1) at (2,0) {$\bullet$};
\draw[-] (0) to (1);
\node (0b) at (0,-2) {\color{white} $\circ$};
\end{tikzpicture}
&
\begin{tikzpicture}[thick]
\node (0) at (0,0) {$\times$};
\node (0b) at (0,-2) {\color{white} $\circ$};
\end{tikzpicture}
&
\begin{tikzpicture}[thick]
\node (0) at (0,0) {$\actiona$};
\node (1) at (2,0) {$\actionb$};
\node (0b) at (0,-2) {\color{white} $\circ$};
\node (0x) at (0.3,-0.5) {\footnotesize{$\neg \Box_\agent \atom$}};
\node (1x) at (2.5,-0.5) {\footnotesize{$\atom \et \Box_\agentb\neg \Box_\agent \atom$}};
\draw[dashed,-] (0) to (1);
\end{tikzpicture}
&
\begin{tikzpicture}[thick]
\node (0) at (0,0) {$=$};
\node (0b) at (0,-2) {\color{white} $\circ$};
\end{tikzpicture}
&
\begin{tikzpicture}[thick]
\node (0) at (0,0) {\framebox{$\circ$}};
\node (1) at (2,0) {$\bullet$};
\node (0b) at (0,-2) {$\circ$};
\draw[-] (0) to (1);
\draw[-,dashed] (0) to (0b);
\end{tikzpicture}
\\ \ \\
\begin{tikzpicture}[thick]
\node (0) at (0,0) {\framebox{$\circ$}};
\node (1) at (2,0) {$\bullet$};
\node (0b) at (0,-2) {$\circ$};
\draw[-] (0) to (1);
\draw[-,dashed] (0) to (0b);
\end{tikzpicture}
&
\begin{tikzpicture}[thick]
\node (0) at (0,0) {$\times$};
\node (0b) at (0,-2) {\color{white} $\circ$};
\end{tikzpicture}
&
\begin{tikzpicture}[thick]
\node (0) at (0,0) {$\actiona$};
\node (1) at (2,0) {$\actionb$};
\node (0b) at (0,-2) {\color{white} $\circ$};
\node (0x) at (0.3,-0.5) {\footnotesize{$\neg \Box_\agent \atom$}};
\node (1x) at (2.5,-0.5) {\footnotesize{$\atom \et \Box_\agentb\neg \Box_\agent \atom$}};
\draw[dashed,-] (0) to (1);
\end{tikzpicture}
&
\begin{tikzpicture}[thick]
\node (0) at (0,0) {$=$};
\node (0b) at (0,-2) {\color{white} $\circ$};
\end{tikzpicture}
&
\begin{tikzpicture}[thick]
\node (0) at (0,0) {\framebox{$\circ$}};
\node (1) at (2,0) {$\bullet$};
\draw[-] (0) to (1);
\node (0b) at (0,-2) {\color{white} $\circ$};
\end{tikzpicture}
\end{array}\]
In the regrettably unpublished \citep{sadzik:2006} the matter of stabilization after action model execution is discussed at great length.

\section{Private lies} \label{sec.six}

\begin{quote}
{\em {\bf True Lies and Butterflies } \\
Mei wants to invite two friends Zhu Yingtai and Liang Shanbo to a party. She knows that they are dying to get close to each other. Thus one will come if and only if (s)he believes that the other will come. Obviously they do not yet wish to admit this to each other, because as far as they know they are both still very uncertain about each other's feelings. Given this uncertainty, both in fact don't intend to come to the party. Mei now lies to Yingtai that Shanbo will come to the party and she lies to Shanbo that Yingtai will come to the party. As a result, they will both come to the party. \\ (This story is a free adaptation of the {\em Butterfly Lovers}, a famous Chinese folktale, see \url{https://en.wikipedia.org/wiki/Butterfly_Lovers}.)
}
\end{quote}

We can consider this an example of a true lie, because when Mei is telling to to Yingtai that Shanbo plans to come, in fact Shanbo is not planning to come, and when she is telling to Shanbo that Yingtai plans to come, in fact Yingtai is not (yet) planning to come. (For modelling convenience we assume that Yingtai is slow in making up her mind after Mei informs her about Shanbo.) After that, they both change their minds, and both lies have become true.

In order to model this story we will depart in two respects from the previous setting of believed announcement logic. In the first place these announcements are not public but private. In the second place the agents that we model change their minds. As we consider the factual propositions `Yingtai plans to go to the party' and `Shanbo plans to go to the party', changing your mind involves factual change. A logic allowing both formalizations is called action model logic (with factual change) \citep{baltagetal:1998,hvdetal.del:2007,jfaketal.lcc:2006}. This logic can be seen as a straightforward generalization of believed announcement logic. (Alternatively, we can model the dynamics in more protocol oriented dynamic epistemic logics, such as \cite{Wang10:phd,hvdetal.aij:2014}.) We only present some examples of action models and assume familiarity with the framework. For technical details, see Appendix A on page \pageref{sec.appendixa} or the above references.

For an initial model, we assume that Yingtai and Shanbo know of themselves whether they intend to go to the party but do not know it of the other one (and that this is known, that this is the background knowledge). This comes with the following model, wherein solid access represents the uncertainty of Yingtai and dashed access represents the uncertainty of Shanbo, and where worlds are named with the facts $p_y$ (`Yingtai comes to the party') and $p_s$ (`Shanbo comes to the party') that are true there, where $01$ stands for `$p_y$ is false and $p_s$ is true', etc. The designated point of the model is boxed: initially both do not intend to go to the party.

\bigskip

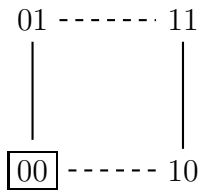
\begin{tikzpicture}[thick]
\node (00m) at (4,0) {\fbox{$00$}};
\node (10m) at (6,0) {$10$};
\node (01m) at (4,2) {$01$};
\node (11m) at (6,2) {$11$};
\draw[dashed,-] (00m) to (10m);
\draw[dashed,-] (01m) to (11m);
\draw[-] (00m) to (01m);
\draw[-] (10m) to (11m);
\end{tikzpicture}

\bigskip

Mei now lies to Yingtai, in private, that Shanbo goes to the party. For the convenience of the reader informed about action model logic, we can model this as a three-pointed action model as follows, on the left---where for convenience we have put the similar private lie to Shanbo about Yingtai next to it, on the right. 

\bigskip

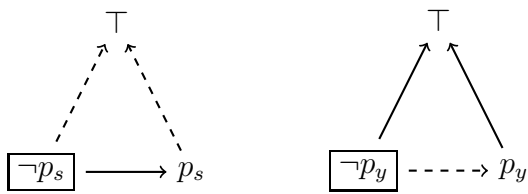
\begin{tikzpicture}[thick]
\node (x) at (0,0) {\fbox{$\neg p_s$}};
\node (y) at (2,0) {$p_s$};
\node (z) at (1,2) {$\top$};
\draw[->] (x) to (y);
\draw[dashed,->] (x) to (z);
\draw[dashed,->] (y) to (z);
\end{tikzpicture}
\quad\quad\quad
\begin{tikzpicture}[thick]
\node (x) at (0,0) {\fbox{$\neg p_y$}};
\node (y) at (2,0) {$p_y$};
\node (z) at (1,2) {$\top$};
\draw[dashed,->] (x) to (y);
\draw[->] (x) to (z);
\draw[->] (y) to (z);
\end{tikzpicture}

\bigskip

The result of the first of these lying actions is as follows.

\bigskip

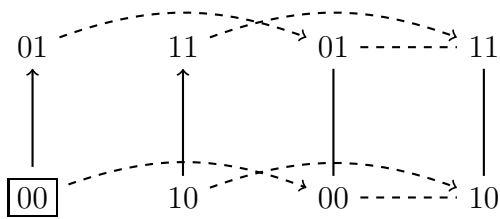
\begin{tikzpicture}[thick]
\node (00m) at (4,0) {\fbox{$00$}};
\node (10m) at (6,0) {$10$};
\node (01m) at (4,2) {$01$};
\node (11m) at (6,2) {$11$};
\node (00r) at (8,0) {$00$};
\node (10r) at (10,0) {$10$};
\node (01r) at (8,2) {$01$};
\node (11r) at (10,2) {$11$};
\draw[dashed,-] (00r) to (10r);
\draw[dashed,-] (01r) to (11r);
\draw[-] (00r) to (01r);
\draw[-] (10r) to (11r);
\draw[dashed,->, bend left=20] (00m) to (00r);
\draw[dashed,->, bend left=20] (01m) to (01r);
\draw[dashed,->, bend left=20] (10m) to (10r);
\draw[dashed,->, bend left=20] (11m) to (11r);
\draw[->] (00m) to (01m);
\draw[->] (10m) to (11m);
\end{tikzpicture}

\bigskip

Now Mei lies privately to Shanbo that Yingtai goes to the party. The result of that action is as follows. For the convenience of the reader we also depict (on the right) the restriction of this model to the submodel generated by the point $00$:

\bigskip

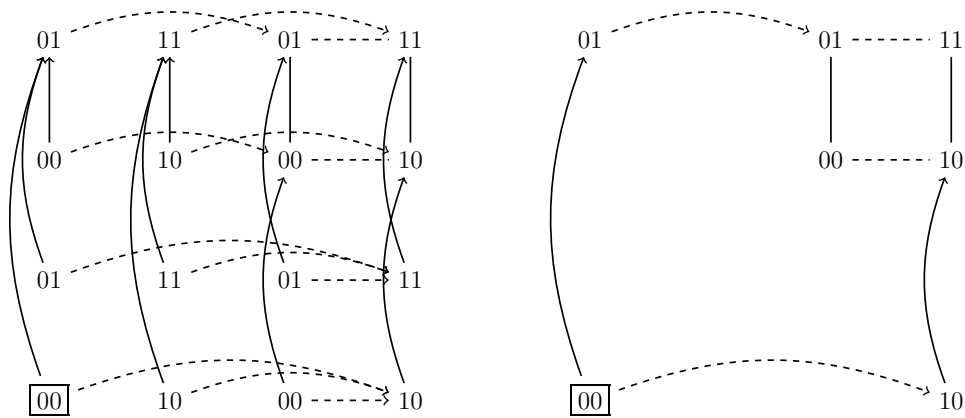
\scalebox{.8}{
\noindent
\begin{tikzpicture}[thick]
\node (00mb) at (4,0) {\fbox{$00$}};
\node (10mb) at (6,0) {$10$};
\node (01mb) at (4,2) {$01$};
\node (11mb) at (6,2) {$11$};
\node (00rb) at (8,0) {$00$};
\node (10rb) at (10,0) {$10$};
\node (01rb) at (8,2) {$01$};
\node (11rb) at (10,2) {$11$};

\node (00m) at (4,4) {$00$};
\node (10m) at (6,4) {$10$};
\node (01m) at (4,6) {$01$};
\node (11m) at (6,6) {$11$};
\node (00r) at (8,4) {$00$};
\node (10r) at (10,4) {$10$};
\node (01r) at (8,6) {$01$};
\node (11r) at (10,6) {$11$};

\draw[dashed,-] (00r) to (10r);
\draw[dashed,-] (01r) to (11r);
\draw[-] (00r) to (01r);
\draw[-] (10r) to (11r);
\draw[dashed,->, bend left=20] (00m) to (00r);
\draw[dashed,->, bend left=20] (01m) to (01r);
\draw[dashed,->, bend left=20] (10m) to (10r);
\draw[dashed,->, bend left=20] (11m) to (11r);
\draw[->] (00m) to (01m);
\draw[->] (10m) to (11m);

\draw[dashed,->] (00rb) to (10rb);
\draw[dashed,->] (01rb) to (11rb);
\draw[->, bend left=20] (00rb) to (00r);
\draw[->, bend left=20] (10rb) to (10r);
\draw[->, bend left=20] (01rb) to (01r);
\draw[->, bend left=20] (11rb) to (11r);
\draw[dashed,->, bend left=20] (00mb) to (10rb);
\draw[dashed,->, bend left=20] (01mb) to (11rb);
\draw[dashed,->, bend left=20] (10mb) to (10rb);
\draw[dashed,->, bend left=20] (11mb) to (11rb);
\draw[->, bend left=20] (00mb) to (01m);
\draw[->, bend left=20] (10mb) to (11m);
\draw[->, bend left=20] (01mb) to (01m);
\draw[->, bend left=20] (11mb) to (11m);
\end{tikzpicture}
\quad\quad\quad\quad
\begin{tikzpicture}[thick]
\node (00mb) at (4,0) {\fbox{$00$}};
\node (10rb) at (10,0) {$10$};
\node (01m) at (4,6) {$01$};

\node (00r) at (8,4) {$00$};
\node (10r) at (10,4) {$10$};
\node (01r) at (8,6) {$01$};
\node (11r) at (10,6) {$11$};

\draw[dashed,-] (00r) to (10r);
\draw[dashed,-] (01r) to (11r);
\draw[-] (00r) to (01r);
\draw[-] (10r) to (11r);
\draw[dashed,->, bend left=20] (00mb) to (10rb);
\draw[->, bend left=20] (00mb) to (01m);
\draw[->, bend left=20] (10rb) to (10r);
\draw[dashed,->, bend left=20] (01m) to (01r);
\end{tikzpicture}
}

\bigskip

The above model formalizes that: Yingtai is not going to the party, believes that Shanbo goes to the party, and believes that Shanbo is uncertain whether she goes to the party; and that: Shanbo is not going to the party, believes that Yingtai goes to the party, and believes that Yingtai is uncertain whether he goes to the party.

\bigskip

We now first let Yingtai change her mind and then Shanbo change his mind. Yingtai changing her mind can again be formalized as an action model, namely as a private assignment to Yingtai; and similarly, Shanbo changing his mind as a private assignment to Shanbo. It is important here that a {\em public} assignment is an improper way to formalize this action: a public assignment of Yingtai going to the party if she believes that Shanbo goes to the party would be informative to Shanbo in case he were to believe that she believed that he was going to the party. Because in case he was uncertain if she would go to the party, he would then learn from this public assignment that she would come to the party for his sake. Boring. Because exactly the absence of this sort of knowledge of the other's intentions makes first lovers' meetings so thrilling. That kind of uncertainty should {\em not} be resolved. Therefore, we formalize it as a private assignment. Interestingly, in the current model the result of a public and of a private assignment (the result of Yingtai privately changing her mind or publicly changing her mind) is the same. But that is because both Yingtai and Shanbo believe that the other is uncertain whether they go to the party.

Below on the left is the action model for Yingtai changing her mind, and on the right, the one for Shanbo changing his mind. (So, for example, according to our conventions, in the left action model the solid relation, that of Yingtai, has identity access, and the dashed relation, for Shanbo, has only a reflexive arrow in the point that the dashed arrow is pointing to.) 

\bigskip

\begin{tikzpicture}[thick]
\node (x) at (0,0) {\fbox{$p_y := \Box_y p_s$}};
\node (y) at (4,0) {$\top$};
\draw[dashed,->] (x) to (y);
\end{tikzpicture}
\quad\quad\quad
\begin{tikzpicture}[thick]
\node (x) at (0,0) {\fbox{$p_s := \Box_s p_y$}};
\node (y) at (4,0) {$\top$};
\draw[->] (x) to (y);
\end{tikzpicture}

\bigskip

We now depict, from left to right, once more the model before they change their minds, the model resulting from executing the action of Yingtai changing her mind, and the model resulting from Shanbo changing his mind, where once again we restrict the actually resulting models to the point-generated subframes.

\bigskip

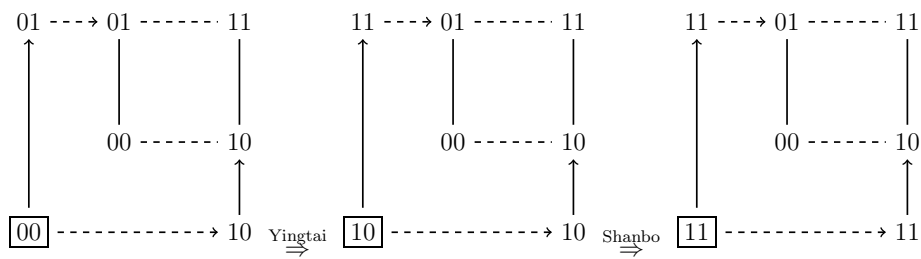
\scalebox{.8}{
\noindent
\begin{tikzpicture}[thick]
\node (00mb) at (6.5,2.5) {\fbox{$00$}};
\node (10rb) at (10,2.5) {$10$};
\node (01m) at (6.5,6) {$01$};

\node (00r) at (8,4) {$00$};
\node (10r) at (10,4) {$10$};
\node (01r) at (8,6) {$01$};
\node (11r) at (10,6) {$11$};

\draw[dashed,-] (00r) to (10r);
\draw[dashed,-] (01r) to (11r);
\draw[-] (00r) to (01r);
\draw[-] (10r) to (11r);
\draw[dashed,->] (00mb) to (10rb);
\draw[->] (00mb) to (01m);
\draw[->] (10rb) to (10r);
\draw[dashed,->] (01m) to (01r);
\end{tikzpicture}
$\stackrel {\text{Yingtai}} \Imp$
\begin{tikzpicture}[thick]
\node (00mb) at (6.5,2.5) {\fbox{$10$}};
\node (10rb) at (10,2.5) {$10$};
\node (01m) at (6.5,6) {$11$};

\node (00r) at (8,4) {$00$};
\node (10r) at (10,4) {$10$};
\node (01r) at (8,6) {$01$};
\node (11r) at (10,6) {$11$};

\draw[dashed,-] (00r) to (10r);
\draw[dashed,-] (01r) to (11r);
\draw[-] (00r) to (01r);
\draw[-] (10r) to (11r);
\draw[dashed,->] (00mb) to (10rb);
\draw[->] (00mb) to (01m);
\draw[->] (10rb) to (10r);
\draw[dashed,->] (01m) to (01r);
\end{tikzpicture}
$\stackrel {\text{Shanbo}} \Imp$
\begin{tikzpicture}[thick]
\node (00mb) at (6.5,2.5) {\fbox{$11$}};
\node (10rb) at (10,2.5) {$11$};
\node (01m) at (6.5,6) {$11$};

\node (00r) at (8,4) {$00$};
\node (10r) at (10,4) {$10$};
\node (01r) at (8,6) {$01$};
\node (11r) at (10,6) {$11$};

\draw[dashed,-] (00r) to (10r);
\draw[dashed,-] (01r) to (11r);
\draw[-] (00r) to (01r);
\draw[-] (10r) to (11r);
\draw[dashed,->] (00mb) to (10rb);
\draw[->] (00mb) to (01m);
\draw[->] (10rb) to (10r);
\draw[dashed,->] (01m) to (01r);
\end{tikzpicture}
}

\bigskip

Now that Yingtai and Shanbo have changed their minds, the lies have become the truth! They both go the party, and they will both expect the other to be surprised to find them at the party.\footnote{This is in accordance with the resulting model, above on the right; for example, in the root \fbox{11} of that model Yingtai believes ---move to the top 11--- that Shanbo believes ---move to 01 and 11--- that Yingtai is uncertain whether Shanbo comes to the party, both when she comes to the party ---the equivalence class consisting of 00 and 01--- as when she does not come to the party ---the cluster consisting of 10 and 11.} They declare their love to each other and they live happily ever after. 

\section{What lies in the future?} \label{sec.last}

We have modelled the true lie $\phi$ in believed announcement logic as the validity 
\[ \neg\phi \imp [\phi]\phi \]
or as the satisfiability of \[ \neg \phi \et [\phi]\phi, \] 
where the believed announcement models {\em informative change} and not factual change. Finding satisfiable true lies, of which there are many, seems a first step towards finding valid true lies, of which there are few. The latter seem to give more insight. They have the shape of correctness statements $\phi \imp [\alpha] \psi$ in Hoare logic; and indeed we have seen that valid true lies come with strong syntactic restrictions. In the private lie example we also modelled actions involving {\em factual change}. With this tool (and action model logic in its full generality) in hand we can also formalize the Pang Juan and the Arnold Schwarzenegger true lies, and make them fit a similar pattern. However, in this case we stop at satisfiability and the mere description of the examples. We defer a more general treatment and generic patterns involving validity to future research.

The Pang Juan true lie can be formalized as a two-pointed action model where one action has precondition $p$, the other action has precondition $\neg p$, where the postcondition is $p := \top$ in both actions (`no matter what was the case, you will now die'), and where the agent only considers the action with precondition $p$ possible. The designated point is the action where $p$ is false (Pang Juan does not die). Let us call this action $\alpha(p)$.

\medskip

\begin{tikzpicture}[thick]
\node (x) at (0,0) {\fbox{$\neg p; p := \top$}};
\node (y) at (4,0) {$p; p := \top$};
\draw[->] (x) to (y);
\end{tikzpicture}

\medskip

Alternatively, we can separate the observation part from the ontic change part, and consider the composition of those two actions. The former is then the {\em lie that $p$}, i.e., the believed announcement that $p$ when $p$ is false. The action model equivalent of a believed announcement is a two-point action model (see Appendix A). The latter a singleton action called public assignment (namely of $\top$ to $p$). So we get:

\medskip

\begin{tikzpicture}[thick]
\node (x) at (0,0) {\fbox{$\neg p$}};
\node (y) at (2,0) {$p$};
\draw[->] (x) to (y);
\end{tikzpicture}
\quad 
\begin{tikzpicture}[thick]
\node (x) at (0,0) {composed with};
\end{tikzpicture}
\quad 
\begin{tikzpicture}[thick]
\node (x) at (0,0) {\fbox{$p := \top$}};
\end{tikzpicture}

\medskip

As the observing agent is Pang Juan and the killing agent one of his opponents, one may prefer the sequence of two epistemic actions over the single epistemic action. On the other hand, as Pang Juan's observation and his subsequent death seem inextricably linked, one might prefer the single epistemic action. However, there is no notion of agency in dynamic epistemic logics: the composition of two actions is indistinguishable from the single action that combines the informative and ontic part: in the logic we have that $[\alpha(p)]\phi$ is equivalent to $[p][p := \top]\phi$, for all $\phi$. So it does not matter. (We also recall that Pang Juan would die no matter which observation he made. We can satisfy this requirement by replacing the announcement that $p$ by the announcement that $q$ in the modelling with two epistemic actions. But that does not give any insight.) We end up with the satisfiability of 
\[ \neg p \et [\alpha(p)]p \] but without being able to generalize this to a validity of shape $\neg p \imp [\alpha(p)]p$. This would require stronger or different logics, involving agency, or other notions of action or interaction \citep{peteretal:2016}, or of {\em self-fulfilling prophesy}.

In the Arnold Schwarzenegger example there is a lot of belief change and factual change going on. The epistemic goals of both agents, and their respective protocols in order to reach those goals, seem to play an important part. 

Let $p_j$ represent `Jamie Lee Curtis is a spy' and $p_a$ represent `Arnold Schwarzenegger is a spy'. At first Arnold is a spy ($p_a$), Arnold believes that he is a spy ($B_a p_a$ ---of course, the reader might say; but hang on!), Jamie believes he is not a spy ($B_j \neg p_a$), and Arnold's goal is to keep it that way ($B_j \neg p_a$). Jamie is not a spy ($\neg p_j$), she believes that she is not a spy ($B_j \neg p_j$), and Arnold also believes this ($B_a \neg p_j$). Jamie's goal also is to keep it that way. So they both believe that the other is not a spy, and this is a stable situation, because both want the other to believe that. We now get a number of transitions where beliefs, facts, and goals all change. The initial situation is (i). \[\begin{array}{l|llllll|l|l}
&&&&&&& \text{goal } a & \text{goal } j \\
\hline
(i) & p_a & \neg p_j & B_a p_a & B_j \neg p_a & B_a \neg p_j & B_j \neg p_j & B_j \neg p_a & B_a \neg p_j \\
(ii) & p_a & \neg p_j & B_a p_a & B_j \neg p_a & B_a \neg p_j & \pmb{B_j p_j} & B_j \neg p_a & \pmb{B_a p_j} \\
(iii) & p_a & \neg p_j & B_a p_a & \pmb{B_j p_a} & B_a \neg p_j & B_j p_j & B_j \neg p_a & B_a p_j \\
(iv) & p_a & \pmb{p_j} & B_a p_a & B_j p_a & \pmb{B_a p_j} & B_j p_j & \pmb{B_j p_a} & B_a p_j \\
\end{array} \]
The epistemic goals of the agents are only realized in (i) and (iv); only there they are stable. Indeed, the last stage is the happy ending. (For details, please see the movie.) In (ii), Jamie has started to enact a spy in order to make Arnold believe that she is a spy; she even believes that she has become a spy, without realizing that she has been set up: she is not really a spy. In (iii), she finds out that Arnold is a spy (because unlike Arnold, third parties believe that she is a spy and that she is collaborating with Arnold). In (iv) ---abstracting from intervening details involving helicopter fights and Miami skyscrapers--- they end up collaborating as spies and believing in each other, now truthfully so. Again, a stable situation.

Let us call the composition of all these actions $\alpha(p_a,p_j)$. Again we have satisfiability of 
\[ (p_a \et \neg p_j) \et [\alpha(p_a,p_j)](p_a \et p_j) \]
but no general pattern involving validity of $(p_a \et \neg p_j) \imp [\alpha(p_a,p_j)](p_a \et p_j)$ or any other additional insight. Lacking in our formal analysis in dynamic epistemic logic are much stronger notions of {\em epistemic protocol} (how does one realize an epistemic goal given a belief state?), protocol logics, and, again, notions of agency \citep{bolanderetal:2011,Wang10:phd}. Second-order false belief scenarios are modelled in \cite{brauneretal:2016,arslanetal:2015}.


\section{Conclusions}

We have presented a large variety of communicative scenarios involving agents truthtelling and lying while keeping their beliefs consistent. Of particular interest were iterations of announcements where the beliefs of the agents never stabilize: the models keep changing with every next announcement. The investigation can be seen as the obvious next step given existing in-depth investigations of successful and self-refuting updates changing agents' knowledge, i.e., on models where epistemic stances are interpreted with equivalence relations. We have carried this forward not merely to the $\kdfv$ models for believing agents, but also to models without structural properties. Two new types of update then come to the fore: true lies (that become true after a lying announcement) and impossible lies (that stay false even after a lying announcement), and also iterations of such announcements. An open question on these iterations is which $\sigma$-valid formulas are realizable.

\section*{Acknowledgements}

Yanjing Wang thanks C\'edric D\'egremont and Andreas Witzel for discussions on the example in Section \ref{sec.five}, and he acknowledges support from the National Program for Special Support of Eminent Professionals. Hans van Ditmarsch acknowledges support from ERC project EPS 313360. He is also affiliated to IMSc (Institute for Mathematical Sciences), Chennai, India. We thank the reviewers of Synthese for their helpful comments.

\bibliographystyle{natbib}
\bibliography{biblio2017}

\section*{Appendix A: Action models} \label{sec.appendixa}

Assume an epistemic model $M = ( \States, R, V )$, and a formula $\phi$. We recall the semantics of {\em believed public announcement} (arrow elimination semantics for public announcement). 
\[ \begin{array}{lcl}
M,\state \models [\phi] \psi &\mbox{iff} & M|\phi,\state \models \psi 
\end{array} \] 
where epistemic model $M|\phi = ( \States, R^\phi, V )$ is as $M$ except that for all $\agent\in\Agents$, $R^\phi_\agent \ := \ R_\agent \inter \ (\States \times \II{\phi}_M)$; and where $\II{\phi}_M := \{ \state\in\States \mid M,\state\models \phi\}$. 

The semantics of {\em truthful public announcement} (state elimination semantics for public announcement) is as follows.
\[ \begin{array}{lcl}
M,\state \models [\phi] \psi &\mbox{iff} & M,\state\models \phi \text{ implies } M|_\phi,\state \models \psi
\end{array} \] 
where $M|_\phi = ( \States', R', V')$ is such that $\States' = \II{\phi}_M$, $R' = R \inter (\II{\phi}_M \times \II{\phi}_M)$, and $V' = V \inter \II{\phi}_M$. 

These different semantics are the same in the following important sense. If the announcement formula is true, the believed announcement semantics and truthful announcement semantics result in bisimilar models. In other words, as bisimilar models have the same logical theory: they cannot be distinguished by a formula in the logic, such as the resulting beliefs of the agents. 

\bigskip

{\em Action model logic} is a generalization of public announcement logic, namely to (possibly) non-public actions. We present the version with factual change. In the language we only have to replace the public announcement modalities $[\phi]\psi$ with action model modalities $[\amodel,\actiona]\psi$, for `after execution of epistemic action $(\amodel,\actiona)$, $\psi$' (is true). An action model is like an epistemic model, only instead of a valuation it has preconditions and postconditions. The syntactic primitive $[\amodel,\actiona]\psi$ may seem a peculiar mix of syntax and semantics, but there is a way to see this as a properly inductive construct in a two-typed language with both formulas and epistemic actions, because the preconditions and postconditions of these actions are again formulas. We now proceed with the technical details.

An {\em action model} $\amodel = (\Actions,\arel,\pre,\post)$ for language $\lang$ (assumed to be simultaneously defined with primitive construct $[\amodel,\actiona]\phi$, see above) consists of a domain $\Actions$ of {\em actions}, an {\em accessibility function} $\arel: \Agents \imp {\mathcal P}(\Actions \times \Actions)$, where each $\arel(\agent)$, for which we write $\arel_\agent$, is an accessibility relation, a {\em precondition function} $\pre: \Actions \imp\lang$, that assigns to each action its executability precondition, and {\em postcondition function} $\post : \Actions \imp \Atoms \not\imp \lang$, where it is required that each $\post(\actiona)$ only maps a finite (possibly empty) subset of all atoms to a formula. For $\actiona \in \Actions$, $(\amodel,\actiona)$ is an {\em epistemic action}.

The semantics of action model execution is as follows.
\[ \begin{array}{lcl}
M,\state \models [\amodel,\actiona] \psi &\mbox{iff} & M,\state\models\pre(\actiona) \text{ implies } M\otimes\amodel,(\state,\actiona) \models \psi 
\end{array} \] 
where $M\otimes\amodel = (\States', R', V')$ (known as {\em update of $M$ with $\amodel$}, or as the {\em result of executing $\amodel$ in $M$}) is such that $S' = \{(\state,\actiona) \mid M,\state\models\actiona\}$; $((\state,\actiona),(\stateb,\actionb)) \in R_\agent$ iff $(\state,\stateb) \in R_\agent$ and $(\actiona,\actionb) \in \arel_\agent$; and $(\state,\actiona) \in V'(\atom)$ iff $M,\state\models\post(\actiona)(\atom)$ for all $\atom$ in the domain of $\post$, and otherwise $(\state,\actiona) \in V'(\atom)$ iff $\state\in V(\atom)$.

A truthful public announcement of $\phi$ corresponds to a singleton action model $\amodel = (\Actions,\arel,\pre,\post)$ with $\Actions = \{\actiona\}$; $\arel_a = \{(\actiona,\actiona)\}$ for all $\agent\in\Agents$; $\pre(\actiona) = \phi$ (and empty domain for postconditions). A believed public announcement of $\phi$ corresponds to a a two-point action model $\amodel = (\Actions,\arel,\pre,\post)$ with $\Actions = \{\actiona,\actionb\}$; $\arel_a = \{(\actionb,\actiona), (\actiona,\actiona)\}$ for all $\agent\in\Agents$; $\pre(\actiona) = \phi$ and $\pre(\actionb) = \neg\phi$ (and again the empty domain for postconditions). If the designated point is $\actionb$ it is a public lie and if the designated point is $\actiona$ it is a truthful (believed) announcement. 

\section*{Appendix B: Depicting knowledge and belief}  \label{sec.appendixb}

To depict epistemic models, a pair of states in the accessibility relation is represented by an arrow between those states. There are visual conventions to simplify the display of multi-agent $\sv$ models and multi-agent $\kdfv$ models. The conventions guarantee that, given a number of agents, any of the following three uniquely determines the two others: (i) the epistemic frame, (ii) the fully displayed visualization, and (iii) the simplified visualization.

In $\sv$ models all accessibility relations are equivalence relations, so that we can partition the domain in equivalence classes: sets of indistinguishable states. In the visualization we can therefore assume transitive, symmetric, and reflexive closure. Visually, indistinguishable states are connected by a path (transitive closure) of undirected edges (symmetric closure). Reflexive arrows not drawn (reflexive closure). Singleton equivalence classes do not show in the visualation. But they are inferable as long as we know how many agents there are.

Examples are as follows. The simplified visualization is on the right and the visualization with all arrows of the accessibility relation is on the left.

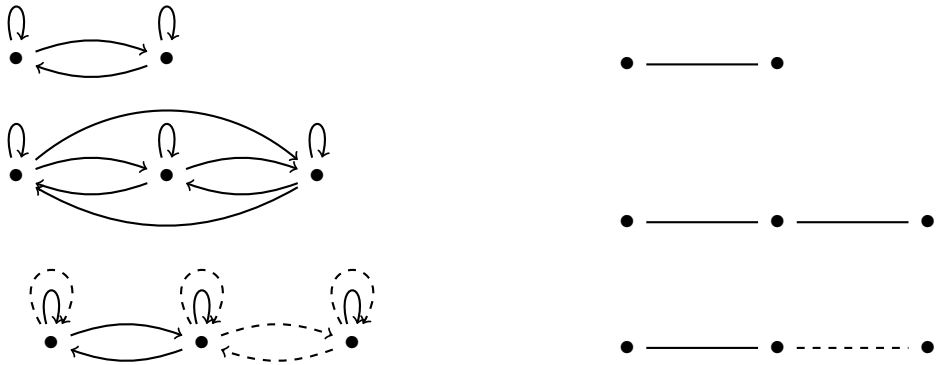
\[\begin{array}{lll}
\begin{tikzpicture}[thick]
\node (0) at (0,0) {$\bullet$};
\node (1) at (2,0) {$\bullet$};
\draw[->,bend left=20] (0) to (1);
\draw[->,bend left=20] (1) to (0);
\draw[->] (0) edge[loop above,looseness=12] (0); 
\draw[->] (1) edge[loop above,looseness=12] (1); 
\end{tikzpicture} & \hspace{2cm} &
\begin{tikzpicture}[thick]
\node (0) at (0,0) {$\bullet$};
\node (1) at (2,0) {$\bullet$};
\draw[-] (0) to (1);
\end{tikzpicture} \\
\begin{tikzpicture}[thick]
\node (0) at (0,0) {$\bullet$};
\node (1) at (2,0) {$\bullet$};
\node (2) at (4,0) {$\bullet$};
\draw[->,bend left=20] (0) to (1);
\draw[->,bend left=20] (1) to (0);
\draw[->,bend left=20] (1) to (2);
\draw[->,bend left=20] (2) to (1);
\draw[->,bend left=40] (0) to (2);
\draw[->,bend left=30] (2) to (0);
\draw[->] (0) edge[loop above,looseness=12] (0); 
\draw[->] (1) edge[loop above,looseness=12] (1); 
\draw[->] (2) edge[loop above,looseness=12] (2); 
\end{tikzpicture} & &
\begin{tikzpicture}[thick]
\node (0) at (0,0) {$\bullet$};
\node (1) at (2,0) {$\bullet$};
\node (2) at (4,0) {$\bullet$};
\draw[-] (0) to (1);
\draw[-] (1) to (2);
\end{tikzpicture} \\
\begin{tikzpicture}[thick]
\node (0) at (0,0) {$\bullet$};
\node (1) at (2,0) {$\bullet$};
\node (2) at (4,0) {$\bullet$};
\draw[->,bend left=20] (0) to (1);
\draw[->,bend left=20] (1) to (0);
\draw[->,dashed,bend left=20] (1) to (2);
\draw[->,dashed,bend left=20] (2) to (1);
\draw[->] (0) edge[loop above,looseness=12] (0); 
\draw[->] (1) edge[loop above,looseness=12] (1); 
\draw[->] (2) edge[loop above,looseness=12] (2); 
\draw[->] (0) edge[loop above,dashed,in=60,out=120,looseness=10] (0); 
\draw[->] (1) edge[loop above,dashed,in=60,out=120,looseness=10] (1); 
\draw[->] (2) edge[loop above,dashed,in=60,out=120,looseness=10] (2); 
\end{tikzpicture} & &
\begin{tikzpicture}[thick]
\node (0) at (0,0) {$\bullet$};
\node (1) at (2,0) {$\bullet$};
\node (2) at (4,0) {$\bullet$};
\draw[-] (0) to (1);
\draw[-,dashed] (1) to (2);
\end{tikzpicture}
\end{array}\]

In $\kdfv$ models all accessibility relations are transitive, euclidean and serial. Therefore we can divide the domain in {\em clusters}. A cluster is a set of indistinguishable states (just as an equivalence class in $\sv$), but those states may also appear to be indistinguishable from the perspective of states outside the cluster (unlike in $\sv$), the so-called {\em unreachable states}. In unreachable states beliefs are incorrect. In the $\kdfv$ visualization we can therefore assume transitive and euclidean closure, and (this is crucial) {\em we can assume reflexive closure of states in a cluster}. This is because a state in a cluster is a state that is reachable from another state. Reflexivity then follows from seriality and euclidicity: from $(x,y) \in R$ and once more) $(x,y) \in R$ follows with euclidicity that $(y,y) \in R)$. It is therefore not so different from the $\sv$ visualization. The only difference is that we also draw arrows from states outside a cluster to a (exactly one) cluster. From the source state of such an arrow is not only accessible the state in the cluster that is the endpoint of that arrow,   but also all other states in that cluster. 

Examples are as follows. 

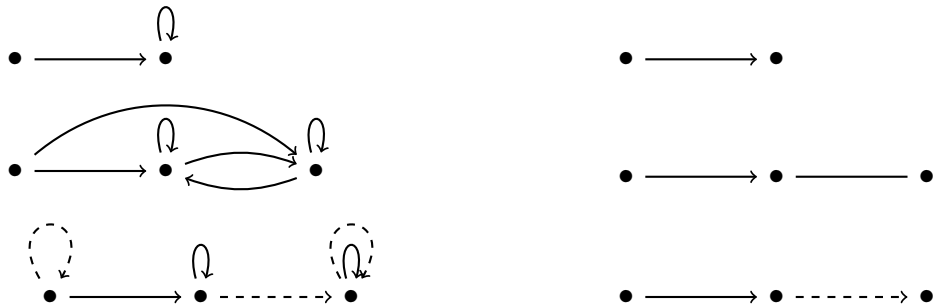
\[\begin{array}{lll}
\begin{tikzpicture}[thick]
\node (0) at (0,0) {$\bullet$};
\node (1) at (2,0) {$\bullet$};
\draw[->] (0) to (1);
\draw[->] (1) edge[loop above,looseness=12] (1); 
\end{tikzpicture} & \hspace{2cm} &
\begin{tikzpicture}[thick]
\node (0) at (0,0) {$\bullet$};
\node (1) at (2,0) {$\bullet$};
\draw[->] (0) to (1);
\end{tikzpicture} \\
\begin{tikzpicture}[thick]
\node (0) at (0,0) {$\bullet$};
\node (1) at (2,0) {$\bullet$};
\node (2) at (4,0) {$\bullet$};
\draw[->] (0) to (1);
\draw[->,bend left=20] (1) to (2);
\draw[->,bend left=20] (2) to (1);
\draw[->,bend left=40] (0) to (2);
\draw[->] (1) edge[loop above,looseness=12] (1); 
\draw[->] (2) edge[loop above,looseness=12] (2); 
\end{tikzpicture} & &
\begin{tikzpicture}[thick]
\node (0) at (0,0) {$\bullet$};
\node (1) at (2,0) {$\bullet$};
\node (2) at (4,0) {$\bullet$};
\draw[->] (0) to (1);
\draw[-] (1) to (2);
\end{tikzpicture} \\
\begin{tikzpicture}[thick]
\node (0) at (0,0) {$\bullet$};
\node (1) at (2,0) {$\bullet$};
\node (2) at (4,0) {$\bullet$};
\draw[->] (0) to (1);
\draw[->,dashed] (1) to (2);
\draw[->] (1) edge[loop above,looseness=12] (1); 
\draw[->] (2) edge[loop above,looseness=12] (2); 
\draw[->] (0) edge[loop above,dashed,in=60,out=120,looseness=10] (0); 
\draw[->] (2) edge[loop above,dashed,in=60,out=120,looseness=10] (2); 
\end{tikzpicture} & &
\begin{tikzpicture}[thick]
\node (0) at (0,0) {$\bullet$};
\node (1) at (2,0) {$\bullet$};
\node (2) at (4,0) {$\bullet$};
\draw[->] (0) to (1);
\draw[->,dashed] (1) to (2);
\end{tikzpicture}
\end{array}\]

A good way to understand the visualization rules is to reproduce the full visualization on the left from the simplified visualization on the right. The role of singleton clusters takes some getting used to. For example, in the first $\kdfv$ example, on the right: As there is an outgoing arrow from the left state, it is not a state in a cluster but an unreachable state. Therefore it has no loop in the completion, on the left. As there is an incoming arrow in the right state it is part of a cluster. This is therefore a singleton cluster. This state must have a loop in the completion, on the left. In the last $\kdfv$ example, as there is an outgoing arrow {\em for the solid relation} from the left state, it is, just as in the first example, not a state in a cluster but an unreachable state. However, as there there is no outgoing arrow {\em for the dashed relation} from the left state, we can infer that it is a singleton cluster and thus infer a loop for the dashed relation. And indeed, on the left, it is there. Etc.

The $\sv$ visualization is well-known. The $\kdfv$ visualization is not well-known.

\end{document}